\documentclass{article}

\usepackage{xr}
\makeatletter
\newcommand*{\addFileDependency}[1]{
  \typeout{(#1)}
  \@addtofilelist{#1}
  \IfFileExists{#1}{}{\typeout{No file #1.}}
}
\makeatother

\newcommand*{\myexternaldocument}[1]{
    \externaldocument{#1}
    \addFileDependency{#1.tex}
    \addFileDependency{#1.aux}
}

\PassOptionsToPackage{numbers, compress}{natbib}

\usepackage[final]{style_files/neurips_2022}

\myexternaldocument{appendix_standalone}

\usepackage{xr-hyper}

\usepackage[utf8]{inputenc} 
\usepackage[T1]{fontenc}    
\usepackage{hyperref}       
\usepackage{url}            
\usepackage{booktabs}       
\usepackage{amsfonts}       
\usepackage{nicefrac}       
\usepackage{microtype}      
\usepackage{xcolor}         

\usepackage{amsmath}
\usepackage{amssymb}
\usepackage{amsthm}

\usepackage{url}
\usepackage{amsfonts}
\usepackage{multicol}
\usepackage{lipsum}
\usepackage{mwe}
\usepackage{subcaption}
\usepackage{multirow}
\usepackage{caption}
\usepackage[english]{babel}

\usepackage{easyeqn}
\usepackage{bbm}
\usepackage[thinlines]{easytable}
\usepackage{algorithm}
\usepackage[algo2e]{algorithm2e}

\usepackage{xcolor}

\usepackage[capitalize,noabbrev]{cleveref}

\usepackage{authblk}

\theoremstyle{plain}
\newtheorem{theorem}{Theorem}[section]

\newtheorem{lemma}[theorem]{Lemma}

\theoremstyle{definition}

\newtheorem{assumption}[theorem]{Assumption}
\theoremstyle{remark}

\def\nsize{N}
\def\RR{\mathbb{R}}
\def\EE{\mathbb{E}}
\def\PP{\mathbb{P}}

\def\uv{\mathbf{u}}

\def\xv{\mathbf{x}}
\def\vx{\mathbf{x}}
\def\vy{\mathbf{y}}

\def\vu{\mathbf{u}}

\def\PPtr{\PP_{\mathrm{train}}}
\def\PPtest{\PP_{\mathrm{test}}}
\def\Dtr{\mathcal{D}} 
\def\risk{\mathcal{R}}
\def\riskp{\mathcal{R}} 
\def\eriskp{\tilde{\mathcal{R}}} 
\newcommand{\dataset}{\mathcal{D}}
\newcommand{\classes}{\mathcal{C}}
\newcommand{\normalDistribution}{\mathcal{N}}

\def\t{\times}
\newcommand{\indicator}{\mathbbm{1}}

\title{Nonparametric Uncertainty Quantification \\ for Single Deterministic Neural Network}

\author[1]{\textbf{Nikita Kotelevskii}}
\author[1]{\textbf{Aleksandr Artemenkov}}
\author[2]{\textbf{Kirill Fedyanin}}
\author[1,3]{\\ \textbf{Fedor Noskov}}
\author[1]{\textbf{Alexander Fishkov}}
\author[6,7]{\textbf{Artem Shelmanov}}
\author[1,6]{\\ \textbf{Artem Vazhentsev}}
\author[4,5]{\textbf{Aleksandr Petiushko}}
\author[2]{\textbf{Maxim Panov}}

\affil[1]{Skolkovo Institute of Science and Technology, Moscow, Russia}
\affil[2]{Technology Innovation Institute, Abu Dhabi, UAE}
\affil[3]{HSE University, Moscow, Russia}
\affil[4]{Nuro, Inc.}
\affil[5]{Lomonosov Moscow State University, Moscow, Russia}
\affil[6]{AIRI, Moscow, Russia}
\affil[7]{Mohamed bin Zayed University of Artificial Intelligence, Abu Dhabi, UAE}

\begin{document}

\maketitle

\begin{abstract}
  This paper proposes a fast and scalable method for uncertainty quantification of machine learning models' predictions. First, we show the principled way to measure the uncertainty of predictions for a classifier based on Nadaraya-Watson's nonparametric estimate of the conditional label distribution. Importantly, the proposed approach allows to disentangle explicitly \textit{aleatoric} and \textit{epistemic} uncertainties. The resulting method works directly in the feature space. However, one can apply it to any neural network by considering an embedding of the data induced by the network. We demonstrate the strong performance of the method in uncertainty estimation tasks on text classification problems and a variety of real-world image datasets, such as MNIST, SVHN, CIFAR-100 and several versions of ImageNet.
\end{abstract}


\section{Introduction}
\label{sec:intro}  
  In many machine learning applications, it is crucial to complement model predictions with uncertainty scores that reflect the degree of trust in these predictions. The total uncertainty of a prediction sums from two uncertainty types, arising from different sources: \textit{aleatoric} and \textit{epistemic}~\citep{der2009aleatory,kendall2017uncertainties}. The former reflects the irreducible noise and ambiguity in the data due to class overlap, while the latter is related to the lack of knowledge about model parameters, and can be reduced by expanding the training dataset. Disentangling epistemic uncertainty can help to identify \textit{out-of-distribution (OOD) data} or to spot instances important for annotation during active learning~\citep{gal2017deep}. The areas with high aleatoric uncertainty might contain some incorrectly labeled instances. If we quantify both types of uncertainty well, we can effectively \textit{abstain from predictions} in unreliable areas and address a decision to a human expert~\citep{el2010foundations}, which is important in safe-critical fields, such as medicine~\citep{miotto2016deep}, autonomous driving~\citep{levinson2011towards,filos2020can}, and finance~\citep{brando2018uncertainty}.
  
  There is no universally accepted uncertainty measure, and diverse, often heuristic treatments are used in practice. For this purpose, one simply could use maximum softmax probabilities of deep neural network (NN). 
  However, MaxProb represents only aleatoric uncertainty, and the resulting measure is notorious to be overconfident in data areas the model did not see during training~\citep{Nguyen2015DeepNN}. Methods based on ensembling \citep{Lakshminarayanan2017SimpleAS} or Bayesian techniques~\citep{Gal2016DropoutAA} can capture both types of uncertainty and yield more reliable uncertainty estimates (UEs)\footnote{Terms uncertainty quantification and uncertainty estimation are often used in literature interchangeably.}. However, along with that, they introduce large computational overhead and might require big modifications to a model architecture and a training procedure.

  Recently, a series of UE methods based on a single deterministic neural network model has been developed~\citep{lee2018simple,van2020uncertainty,liu2020simple,van2021improving}. The main idea behind these methods is to leverage geometrical proximity of instances in a vector space of neural network hidden representations that capture semantic relationships between instances in the input space. Such methods are very promising since they are computationally efficient and usually do not require big changes in network architectures and training procedures, which gives them versatility.
  
  Many of these methods rely on the assumption that (conditionally on the class) training instances are distributed normally in the latent space of a trained neural network~\citep{van2021improving,lee2018simple}. However, this assumption does not always hold. For example, a border between classes might be more complex, as well as the border between the in-domain and the out-of-domain region. This happens when training is accomplished in a low-resource regime with small amounts of labeled instances.
  Another problem is that some methods are hard to train and the model performance can substantially deteriorate without proper hyperparameter tuning. Finally, some methods require computing the covariance matrix of training data~\cite{lee2018simple}. This operation might be computationally unstable if the training dataset is not large enough.

  \begin{figure}[t!]
    \centering
    \includegraphics[width=\textwidth]{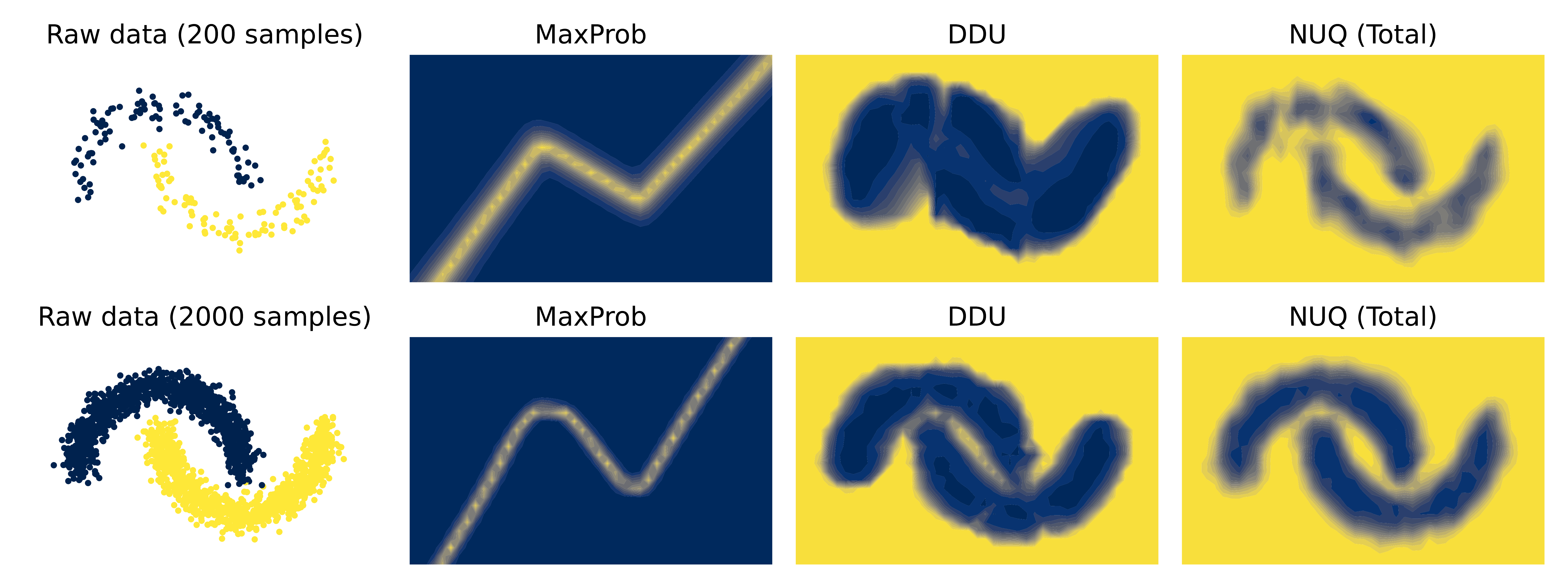}
    \caption{The first column on the left shows the raw data from the Two Moons dataset. The second column illustrates neural network prediction uncertainty obtained using MaxProb. The last two columns illustrate the results of DDU and the sum of aleatoric and epistemic uncertainties obtained using NUQ. The lighter color indicates higher uncertainty. 
    }
    \label{fig:toy_results_1}
  \end{figure}
  
  In this work, we propose a new principled approach to UE that overcomes these limitations and demonstrates very robust performance both in low-resource and in high-resource regimes. We suggest to look at the pointwise Bayes risk (probability of the wrong prediction) as the natural measure of the model prediction uncertainty at a particular data point. Then we consider the Nadaraya-Watson estimator of the conditional label distribution. Its asymptotic Gaussian approximation allows deriving uncertainty estimate based on the upper bound for the risk. The resulting Nonparametric Uncertainty Quantification (NUQ) method is amenable for uncertainty disentanglement (into aleatoric and epistemic) and is implemented in a scalable manner, which allows it to be used on large datasets such as ImageNet. 
  
  To illustrate the strong points of the proposed method, we suggest to look at a toy example in Figure~\ref{fig:toy_results_1}. Here, we compare NUQ with MaxProb and DDU~\cite{mukhoti2021deterministic}, one of the strong baselines in deterministic uncertainty estimation aimed at epistemic uncertainty. MaxProb as a measure of aleatoric uncertainty gives high values only in-between classes while being very confident far from the data. We also see that for a larger data set of 2000 points NUQ and DDU show similar results. However, when the neural network model is trained only on 200 labeled instances, DDU scores become very rough. It cannot spot an unreliable area in-between classes, while the NUQ-based total uncertainty measure does it precisely with high uncertainty score.
  Importantly, the presented total uncertainty for NUQ can be also disentangled into aleatoric and epistemic ones.

  Below, we give a \textbf{summary of contributions} of this paper.
  \begin{enumerate}
    \item We develop \textit{a new and theoretically grounded} method for nonparametric uncertainty quantification (see Section~\ref{sec:method}), which is (a) applicable to any deterministic neural network model, (b) has an efficient implementation, (c) allows to disentangle aleatoric and epistemic uncertainty, (d) outperforms other techniques on image and text classification tasks (see Section~\ref{sec:experiments}).
  
    \item We formally prove the consistency of the NUQ-based estimator when applied to the problem of classification with the reject option, see Section~\ref{sec:consistency}.

    \item We conduct a vast empirical investigation of NUQ and other UE techniques on image and text classification tasks that supports our claims, see Section~\ref{sec:experiments}. 
  \end{enumerate}


\section{Nonparametric Uncertainty Quantification}
\label{sec:method}

\subsection{Classification under Covariate Shift}
  In this section and below, for the sake of clarity, we provide derivations for the binary case. We derive a generalization to a multi-class setting in Supplementary Material (SM), Section~\ref{sec:multiclass_generalization}.
  
  Let us consider the standard binary classification setup \((X, Y) \in \RR^d \t \{0, 1\}\) with \((X, Y) \sim \PP\), where \(\PP\) is a joint distribution of objects and the corresponding labels. We assume that we observe the dataset \(\Dtr = \bigl\{(\xv_i, y_i)\}_{i = 1}^{\nsize}\) of \textit{i.i.d.} points from \(\PP = \PPtr\). Here, \(\xv\) and \(y\) are realizations of random variables \(X\) and \(Y\).

  The classical problem in statistics and machine learning is using the dataset \(\Dtr\), to find a rule \(\hat{g}\), which approximates the optimal one:
  
  \vspace{-24pt}

  \begin{EQA}[c]
    g^* = \arg\min_{g} \PP(g(X) \neq Y).
  \end{EQA}
  
  \vspace{-14pt}

  Here, \(g\colon \RR^d \to \{0, 1\}\) is any classifier, and the probability of wrong classification \(\PP(g(X) \neq Y)\) is usually called \textit{risk}.
  The rule \(g^*\) is given by the \textit{Bayes optimal classifier}:
  
  \vspace{-24pt}

  \begin{EQA}[c]
    g^*(\xv) = 
    \begin{cases}
      1, &\eta(\xv) \geq \frac{1}{2}, \\
      0, &\eta(\xv) < \frac{1}{2},
    \end{cases}
    \label{g_def}
  \end{EQA}

  \vspace{-14pt}

  where \(\eta(\xv) = \PP(Y = 1 \mid X = \xv)\) which is the conditional probability of \(Y\) given \(X = \xv\) under the distribution \(\PP\).

  In this work, we consider a situation when the distribution of the test samples \(\PPtest\) is different from the one for the training dataset \(\PPtr\), i.e. \(\PPtest \neq \PPtr\). Obviously, the rule \(g^*\) obtained for \(\PP = \PPtr\) might no longer be  optimal if the aim is to minimize the error on the test data \(\PPtest(g(X) \neq Y)\).

  In order to formulate a meaningful estimation problem, some additional assumptions are needed. We assume that the conditional label distribution \(\eta(\xv) = \PP(Y \mid X = \xv)\) \textbf{is the same} under both \(\PPtr\) and \(\PPtest\). This assumption has two important consequences:

  \vspace{-5pt}

  \begin{enumerate}
    \item The entire difference between \(\PPtr\) and \(\PPtest\) is due to the difference between marginal distributions of \(X\): \(p_{\mathrm{train}}(X)\) and \(p_{\mathrm{test}}(X)\). The situation when \(p_{\mathrm{test}}(X) \neq p_{\mathrm{train}}(X)\) is known as \textit{covariate shift}.

    \item The rule \(g^*\) is still valid, i.e., optimal under \(\PPtest\).
  \end{enumerate}

  \vspace{-5pt}

  However, while the classifier \(g^*\) is still optimal under covariate shift, its approximation \(\hat{g}\) might be arbitrarily bad. The reason for that is that we cannot expect \(\hat{g}\) to approximate \(g^*\) well in the areas where we have few objects from the training set or do not have them at all. Thus, some special treatment of the covariate shift is required.

\subsection{Pointwise Risk and Its Estimation}
\label{sec:pointwise_risk}
  We consider a classification rule \(\hat{g}(X) = \hat{g}_{\Dtr}(X)\) constructed based on the dataset \(\Dtr\). Let us start from defining the pointwise risk of a prediction:
  
  \vspace{-24pt}

  \begin{EQA}[c]
    \riskp(\xv) = \PP(\hat{g}(X) \neq Y \mid X = \xv),
  \end{EQA}

  \vspace{-14pt}

  where \(\PP(\hat{g}(X) \neq Y \mid X = \xv) \equiv \PPtr(\hat{g}(X) \neq Y \mid X = \xv) \equiv \PPtest(\hat{g}(X) \neq Y \mid X = \xv)\) under the assumptions above. The value \(\riskp(\xv)\) is independent of the covariate distribution \(p_{\mathrm{test}}(X)\) and essentially allows to define a meaningful target of estimation, which is based solely on the quantities known for the training distribution.

  Let us note that the total risk value \(\riskp(\xv)\) admits the following decomposition:
  
  \vspace{-24pt}

  \begin{EQA}[c]
    \riskp(\xv) = \eriskp(\xv) + \risk^{*}(\xv),
  \end{EQA}

  \vspace{-14pt}

  where \(\risk^{*}(\xv) = \PP(g^*(X) \neq Y \mid X = \xv)\) is the Bayes risk and \(\eriskp(\xv) = \PP(\hat{g}(X) \neq Y \mid X = \xv) - \PP(g^*(X) \neq Y \mid X = \xv)\) is an excess risk.
  Here, $\risk^{*}(\xv)$ corresponds to aleatoric uncertainty as it completely depends on the data distribution. The excess risk $\eriskp(\xv)$ directly measures imperfectness of the model \(\hat{g}\) and, thus, can be seen as a measure of epistemic uncertainty.

  To proceed, we first assume that the classifier \(\hat{g}\) has the standard form:
  
  \vspace{-24pt}

  \begin{EQA}[c]
    \hat{g}(\xv) = 
    \begin{cases}
      1, &\hat{\eta}(\xv) \geq \frac{1}{2}, \\
      0, &\hat{\eta}(\xv) < \frac{1}{2},
    \end{cases}
    \label{hat_g_def}
  \end{EQA}

  \vspace{-14pt}

  where \(\hat{\eta}(\xv) = \hat{p}(Y = 1 \mid X = \xv)\) is an estimate of the conditional density \(\eta(\xv)\).

  For such an estimate, we can upper bound the excess risk via the following classical inequality~\cite{devroye2013probabilistic}:
  
  \vspace{-24pt}

  \begin{EQA}[c]
    \eriskp(\xv) = \PP(\hat{g}(X) \neq Y \mid X = \xv) - \PP(g^*(X) \neq Y \mid X = \xv)
    \leq 2|\hat{\eta}(\xv) - \eta(\xv)|.
  \label{eq:main_ineq}
  \end{EQA}

  \vspace{-14pt}

  It allows us to obtain an upper bound for the total risk:
  
  \vspace{-24pt}

  \begin{EQA}[c]
    \mathcal{R}(\xv) \leq \mathcal{L}(\xv) = \risk^{*}(\xv) + 2  |\hat{\eta}(\xv) - \eta(\xv)|,
  \label{eq:surrogate_loss_ood}
  \end{EQA}

  \vspace{-14pt}

  where \(\risk^{*}(\xv) = \min\{\eta(\xv), 1 - \eta(\xv)\}\) in the case of binary classification. While this upper bound still depends on the unknown quantity \(\eta(\xv)\), we will see in the next section that \(\mathcal{L}(\xv)\) allows for an efficient approximation under mild assumptions.

\subsection{Nonparametric Uncertainty Quantification}
\label{sec:nuq_main}
\subsubsection{Kernel Density Estimate and Its Asymptotic Distribution}
  To obtain an estimate of \(\mathcal{L}(\xv)\) and, consequently, bound the risk, we need to consider some particular type of estimator \(\hat{\eta}\). In this work, we choose the classical kernel-based Nadaraya-Watson estimator of the conditional label distribution as it allows for a simple description of its asymptotic properties.

  Let us denote by $K_h\colon\mathbb{R}^d \mapsto\mathbb{R}$ the multi-dimensional kernel function with bandwidth \(h\). Typically, we consider a  multi-dimensional Gaussian kernel, but other choices are also possible.

  The conditional probability estimate is expressed as ($y_i$ is either 0 or 1):

  \vspace{-14pt}

  \begin{equation}
    \hat{\eta}(\xv) =
    \frac{\sum_{i=1}^N \indicator[y_i = 1] \cdot  K_h(\xv - \xv_i)}{\sum_{i=1}^N K_h(\xv - \xv_i)}.
  \label{eq:class_conditional}
  \end{equation}
  
  \vspace{-5pt}

  The difference between $\hat{\eta}(\xv) - \eta(\xv)$ for properly chosen bandwidth \(h\) converges in distribution as follows (see, e.g.~\cite{powellreg}):
  
  \vspace{-14pt}

  \begin{equation}
    \hat{\eta}(\xv) - \eta(\xv) \rightarrow \mathcal{N}\biggl(0,\frac{\tilde{C}}{N} \frac{\sigma^2(\xv)}{p(\xv)}\biggr),
    \label{gaussian_asymptotic}
  \end{equation}
  
  \vspace{-5pt}
  
  where $\nsize$ is the number of data points in the training set, $p(\xv)$ is the marginal distribution of covariates (see details in SM, Section~\ref{appendix:hyperparams_kde}), and $\sigma^2(\xv)$ is the standard deviation of the data label at point $\xv$. For binary classification, $\sigma^2(\xv) = \eta(\xv)\bigl(1 - \eta(\xv)\bigr)$. The constant $\tilde{C}=\int [K_h(\vu)]^2 d\vu$, where $\uv$ is an integration variable, depends only on the choice of the kernel $K_h$ and could be computed in closed form for popular kernels. See in details in SM, Section~\ref{appendix:hyperparams_kernels}.
  
  Now, we are equipped with an estimate of the distribution for \(\hat{\eta}(\xv) - \eta(\xv)\). Let us denote by $\tau(\xv)$ the standard deviation of a Gaussian from the  equation~\eqref{gaussian_asymptotic}:
  
  \vspace{-24pt}

  \begin{EQA}[c]
    \tau^2(\xv) = \frac{\tilde{C}}{N} \frac{\sigma^2(\xv)}{p(\xv)}.
  \end{EQA}
  
  \vspace{-14pt}

  In the following sections, we first show how to use the obtained property for uncertainty estimation, and then, we show how it can be computed.

\subsubsection{Total, Aleatoric, and Epistemic Uncertainty and Their Estimates}
  In this work, we suggest a particular uncertainty quantification procedure inspired by the derivation above, which we call \textit{Nonparametric Uncertainty Quantification (NUQ)}. More specifically, we suggest to consider the following measure of the total uncertainty:
  
\vspace{-24pt}

\begin{EQA}[c]
    \mathbf{U}_t(\xv) = \min\bigl\{\eta(\xv), 1 - \eta(\xv)\bigr\} + 2 \sqrt{\frac{2}{\pi}} \tau(\xv).
  \label{eq:total_un_def}
  \end{EQA}
  
  \vspace{-14pt}

  This measure is obtained by considering an asymptotic approximation of the expected value of the total risk upper bound:
  
  \vspace{-24pt}

  \begin{EQA}[c]
    \EE_{\Dtr} \mathcal{L}(\xv) = \min\bigl\{\eta(\xv), 1 - \eta(\xv)\bigr\} + 2 \EE_{\Dtr} \bigl|\hat{\eta}(\xv) - \eta(\xv)\bigr|
  \end{EQA}
  
  \vspace{-14pt}

  in view of~\eqref{gaussian_asymptotic} and the fact that \(\EE |\xi| = \text{std}(\xi)\sqrt{\frac{2}{\pi}}\) for the zero-mean normal variable \(\xi\). The resulting estimate upper bounds the average error of estimation at the point \(\xv\) and thus indeed can be used as the measure of total uncertainty. 

  We also can write the corresponding measures of aleatoric and epistemic uncertainties:
  
  \vspace{-24pt}

  \begin{EQA}[c]
    \mathbf{U}_a(\xv) = \min\bigl\{\eta(\xv), 1 - \eta(\xv)\bigr\}, \quad
    \mathbf{U}_e(\xv) = 2 \sqrt{\frac{2}{\pi}} \tau(\xv). \quad
  \label{eq:epistemic_un_def}
  \end{EQA}
  
  \vspace{-14pt}

  Finally, the data-driven uncertainty estimates \(\hat{\mathbf{U}}_a(\xv) , \hat{\textbf{U}}_e(\xv)\) and \(\hat{\textbf{U}}_t(\xv)\) can be obtained via plug-in using estimates \(\hat{\eta}(\xv)\), \(\hat{\sigma}(\xv)\), \(\hat{p}(\xv)\) and, consequently, \(\hat{\tau}^2(\xv) = \frac{1}{N} \frac{\hat{\sigma}^2(\xv)}{\hat{p}(\xv)} \tilde{C}\).

  We should note that despite being based on asymptotic approximation, the resulting formulas for uncertainties are very natural and make sense for finite sample size (for example, epistemic uncertainty is proportional to $\sigma^2(\xv) / p(\xv)$). We also should note that even known non-asymptotic decompositions for the risk of the NW-estimator still contain the same $\sigma^2(\xv) / p(\xv)$ component, see Proposition 1 in~\cite{chagny2016adaptive}. That means that the proposed estimate well captures the general uncertainty trend.

\paragraph{Efficient computation.}
  We note that computation of the nonparametric estimate~\eqref{eq:class_conditional} involves a sum over the whole available data. This could be intractable in practice when we are working with large datasets. However, the typical kernel $K_h$ quickly approaches zero with the increase of the norm of the argument: $\|\xv - \xv_i\|$. Thus, we can use an approximation of the kernel estimate: instead of the sum over all elements in the dataset, we consider the contribution of only several nearest neighbors (see SM, Section~\ref{appendix:hyperparams_nn} for details). It requires a fast algorithm for finding the nearest neighbors. For this purpose, we use the approach of~\cite{malkov2018efficient} based on Hierarchical Navigable Small World graphs (HNSW). It provides a fast, scalable, and easy-to-use solution to the computation of nearest neighbors.

\paragraph{Application to NN and Comparison with Existing Methods.}
  The resulting NUQ method can be applied to NN in the postprocessing fashion, i.e. one can fit it on top of the embeddings of the trained NN model. One may wonder about the difference between NUQ and other embedding based methods such as, for example, DUQ~\cite{van2020uncertainty} or DDU~\cite{mukhoti2021deterministic}. The difference is twofold: (i) NUQ is based on the rigorous derivation of total, aleatoric, and epistemic uncertainties, while other methods usually consider more heuristic treatment and do not allow for uncertainty disentanglement; (ii) NUQ considers a more flexible estimator of density in the embedding space that allows to achieve better quality in a small training data regime or for complicated data; see experimental evaluation in Section~\ref{sec:experiments}.

\section{Detailed Algorithmic Description of NUQ Approach}
  In this section, we provide a detailed algorithmic description of the NUQ approach.
  On the training stage, NUQ uses training data embeddings obtained from the pre-trained neural network and builds a Bayesian classifier based on conditional label probabilities estimated in a non-parametric way:
  \begin{EQA}[c]
    \hat{p}(Y = c \mid X = \xv) =
    \frac{\sum_{i=1}^N \indicator[y_i = c] \cdot K_h(\xv - \xv_i)}{\sum_{i=1}^N K_h(\xv - \xv_i)}, ~~ c = 1, \dots, C.
  \label{eq:class_conditional_multi}
  \end{EQA}
  The bandwidth $h$ is tuned via cross-validation optimizing the classification accuracy on the training data (see SM, Section~\ref{appendix:hyperparams_bandwidth}). 

  On the inference stage, the new object is passed through the neural network and a corresponding embedding is computed. Then, this embedding is used to compute the uncertainty estimates employing the estimated bandwidth $h$, see Algorithm~\ref{algo:nuq} that details the computation of all the necessary intermediate quantities as well as the resulting uncertainty estimates based on the embeddings $\xv$ provided by the neural network. Algorithm~\ref{algo:nuq} also takes into account the usage of nearest neighbours to speed up the computation of kernel-based estimates. Finally, the computation of an estimate of the embeddings density $\hat{p}(\xv)$ can be done either via the kernel density estimate (KDE) or via the Gaussian Mixture Model (GMM). We study the relative benefits of these approaches in Section~\ref{sec:ablation}.

  \begin{algorithm}[t]
    \KwIn{Training set $\{(\xv_i, y_i)\}_{i=1}^N$,  inference point $\xv$, bandwidth $h$}
    \KwOut{Prediction $\hat{g}(\xv)$ and uncertainty estimate $\hat{\mathbf{U}}_t(\xv)$}
    \BlankLine
    
    $\{\xv_{i_k}\}_{k=1}^K \gets K$ nearest neighbors of $\xv$ among $\{\xv_i\}_{i=1}^N$
    
    $ \hat{p}(Y = c \mid X = \xv) \gets \frac{\sum_{k=1}^K K_h(\xv_{i_k} - \xv)\indicator[y_{i_k}=c]}{\sum_{k=1}^K K_h(\xv_{i_k} - \xv)}$
    
    $\hat{\sigma}_c^2(\xv) = \hat{p}(Y = c \mid X = \xv) \bigl(1 - \hat{p}(Y = c \mid X = \xv)\bigr)$
    
    $\hat{g}(\xv) \gets \underset{c}{\operatorname{argmax}} \;\; \hat{p}(Y = c \mid X = \xv)$
    
    $\hat{p}(\xv) \gets \text{either KDE: } \frac{1}{Nh^d} \sum_{k=1}^K K_h(\xv_{i_k} - \xv) \text{ or GMM}$ 
    
    $\hat{\tau}^2(\xv) \gets \frac{1}{N} \frac{\underset{c}{\max} \; \hat{\sigma}_c^2(\xv)}{\hat{p}(\xv)} \tilde{C}, \text{where } \tilde{C}=\int [K_h(\vu)]^2 d\vu$ \;
    
    
    $\hat{\mathbf{U}}_t(\xv) \gets \underset{c}{\min} \bigl\{1 - \hat{p}(Y = c \mid X = \xv)\bigr\} + 2 \sqrt{\frac{2}{\pi}} \hat{\tau}(\xv)$
    \caption{NUQ inference algorithm.}
    \label{algo:nuq}
  \end{algorithm}

\section{Consistency of NUQ-based Classification with a Reject Option}
\label{sec:consistency}
  Above, we obtained uncertainty estimates that characterize the classical risk of a prediction. However, they are also helpful to solve the formal problem of classification with the reject option. In this problem, for any input \(\xv\) we can choose either we perform prediction or reject it. Following~\cite{chow_optimum_1970}, we assume that in the case of prediction we pay a binary price depending whether the prediction was correct or not, while in the case of the rejection we pay the constant price \(\lambda \in (0, 1)\). For this task, the risk function is
  \begin{align*}
    \risk_{\lambda}(\xv) = \risk(\xv) \indicator\{\alpha(\xv) = 0\} + \lambda \indicator\{\alpha(\xv) = 1\},
  \end{align*}
  where \(\alpha(\xv)\) is an indicator of the rejection.

  The minimizer of \(\risk_{\lambda}(\xv)\) is given by the optimal Bayes classifier \(g^*(\xv)\) and the abstention function
  \begin{align*}
    \alpha^*(\xv) =
    \begin{cases}
      0, & \risk^*(\xv) \leq \lambda, \\
      1, & \risk^*(\xv) > \lambda.
    \end{cases}
  \end{align*}
  To approximate \(\alpha^*(\xv)\), we utilize hypothesis testing:
  \begin{align*}
    H_0\colon \risk^*(\xv) > \lambda \text{ vs. } H_1\colon \risk^*(\xv) \leqslant \lambda.
  \end{align*}
  We choose the confidence level \(\beta > 0\) and consider the statistic
  $  \hat{\mathbf{U}}_{\beta}(\xv) = \min\bigl\{\hat{\eta}(\xv), 1 - \hat{\eta}(\xv)\bigr\} + z_{1 - \beta} \hat{\tau}(\xv),
  $
  where \(z_{1 - \beta}\) is the \(1 - \beta\) quantile of the standard normal distribution. The statistic \(\hat{\mathbf{U}}_{\beta}(\xv)\) combines the plug-in estimate of the Bayes risk \(\min\bigl\{\hat{\eta}(\xv), 1 - \hat{\eta}(\xv)\bigr\}\) and the term \(z_{1 - \beta} \hat{\tau}(\xv)\) accounting for the confidence of estimation. The resulting abstention rule is given by:
  
  \vspace{-24pt}

  \begin{EQA}[c]
    \hat{\alpha}_{\beta}(\xv) =
    \begin{cases}
      0, & \hat{\mathbf{U}}_{\beta}(\xv) \leq \lambda, \\
      1, & \hat{\mathbf{U}}_{\beta}(\xv) > \lambda.
    \end{cases}
  \end{EQA}
  
  \vspace{-14pt}

  Finally, if we consider the pair of kernel classifier \(\hat{g}(\xv)\) and \(\hat{\alpha}_{\beta}(\xv)\) then, we can prove the consistency result for the corresponding risk \(\hat{\risk}_\lambda(\xv)\) under standard assumptions on nonparametric densities, see SM, Section~\ref{sec:consistency_proof} for details.
  \begin{theorem}
  \label{theorem: asymptotic consistency}
    Suppose that assumptions~\ref{assumption: regression funciton}-\ref{assumption: kernel} hold and $p(\xv) > 0$, the bandwidth $h \to 0$ and $\nsize h^d \to \infty$ as $\nsize$ tends to infinity. Then, for any $\beta < 1/2$:
    \begin{align*}
      \EE_\dataset \hat{\risk}_\lambda(\xv) - \risk^*_\lambda(\xv) \underset{\nsize \to \infty}{\longrightarrow} 0.
    \end{align*}
    %
  \end{theorem}

  This result shows the validity of the NUQ-based abstention procedure. Interesting future work is to obtain a precise convergence rate for the method. It should be possible based on the finite sample bounds provided in SM, Section~\ref{sec:consistency_proof}. We also experimentally illustrate the benefits of the proposed estimator in SM, Section~\ref{sec:reject_example}.


\section{Related Work}
\label{sec:related_work}
  The notion of uncertainty naturally appears in Bayesian statistics~\citep{gelman2013bayesian}, and, thus, Bayesian methods are often used for uncertainty quantification. The exact Bayesian inference is computationally intractable, and approximations are used. Two popular ideas are the Markov Chain Monte Carlo sampling (MCMC; \cite{neal2011mcmc}) and the Variational Inference (VI; \cite{blei2017variational}). MCMC has theoretical guarantees to be asymptotically unbiased, but it has a high computational cost. VI-based approaches~\citep{rezende2015variational,dinh2016density,papamakarios2019normalizing,kobyzev2020normalizing} are more scalable, they are biased and at least double the number of parameters. That is why some alternatives are considered, such as the Bayesian treatment of Monte-Carlo dropout~\citep{Gal2016DropoutAA}. 

  Deep Ensemble~\citep{Lakshminarayanan2017SimpleAS} is usually considered as a quite strong yet expensive approach. A series of papers developed ways of approximating the distribution obtained using an ensemble of models by a single probabilistic model~\citep{malinin2018predictive,malinin2019ensemble, sensoy2018evidential}. These methods require changing the training procedure and need more parameters to train.

  Recently, a series of uncertainty quantification approaches for a single deterministic neural network model was proposed. In DUQ~\citep{van2020uncertainty}, an RBF layer is added to the network with a custom training procedure to adjust the centroid points (in the embedding space). The downside of the method is its inability to distinguish aleatoric and epistemic uncertainty. Another approach to capture epistemic uncertainty was proposed in DDU~\citep{mukhoti2021deterministic}. It uses a Gaussian mixture model to estimate the density of objects in the embedding space of a trained neural network. The density values are then used as a confidence measure. SNGP~\citep{liu2020simple} and DUE~\citep{van2021improving} are similar but use a Gaussian process as the final layer, requiring estimating covariance with the use of inducing points or RFF expansion. 

  There is a wide range of papers discussing classification with the reject option. Most likely, the problem was firstly studied by Chow in~\cite{chow_optimum_1957, chow_optimum_1970}. Moreover, in ~\cite{chow_optimum_1970}, he introduced a risk function used across this paper. Herbei et al. \cite{herbei_classification_2006} studied an optimal procedure for this risk and provided a plug-in rule. In the following works, empirical risk minimization among a class of hypotheses (see~\cite{bartlett_classification_2008, cortes_learning_2016}) or other types of risk (see~\cite{denis_consistency_2015, el-yaniv_foundations_2010, lei_classification_2014}) were investigated. Besides, a number of practical works have been presented, see, for example,~\cite{grandvalet_support_2009, geifman_selectivenet_2019, nadeem_accuracy-rejection_2009}.


\section{Experiments}
\label{sec:experiments}
  We conduct a series of experiments on image and text classification datasets. In each experiment, (1) we train a parametric model -- a neural network, which we call a \textit{base model}; (2) fit NUQ on the training data using the embeddings obtained from the base model. We use logits as extracted features, if not explicitly stated otherwise. However, other options are also possible; see Section~\ref{sec:ablation}.

  Following SNGP and DDU, we use spectral normalization to train the base model to achieve bi-Lipschitz property and avoid the feature collapse and non-smoothness. However, NUQ works sufficiently good even without this regularization (see Table~\ref{tab:spectral} in Section~\ref{sec:ablation}). In all experiments on OOD detection, we use \(\hat{\textbf{U}}_e(\vx)\) as a measure of uncertainty. An additional illustrative experiment on detecting actual aleatoric and epistemic uncertainties is presented in SM, Section~\ref{appendix:toy_example_actual_uncertainties}.

\subsection{How NUQ Affects Model Predictions?} 
  One may ask whether the nonparametric classification method used in NUQ, trained on some embedding from the base model, has any relation to the original neural network. To reassure the reader, we provide an argument that it well approximates the predictions of the base model and NUQ-based uncertainty estimates can be used for the base model as well. Specifically, we compute the agreement 
  between predictions obtained from the Bayes classifier based on kernel estimate (i.e. the one used in NUQ) and base models' predictions. This metric formally can be defined as
  $
    \text{agreement}(\hat{p}, p) = \frac{1}{n} \sum_{i=1}^n I \bigl[\arg\max_j \hat{p}(y=j \mid \vx_i) = \arg\max_j p(y=j \mid \vx_i)\bigl].
  $ 
  For CIFAR-100 (see experiments with this dataset in Section~\ref{sec:cifar100_exps}), this metric gives us the agreement of 0.975, which shows that the approach is accurate. Additionally, we computed the aleatoric uncertainty for all test objects and found the average percentile of the objects with disagreement is 94.94 $\pm$ 4.75. Thus, disagreement appears for high uncertainty points as expected.

\subsection{Image Classification}
  The main experiments with image classification are conducted on CIFAR-100~\citep{cifar} and ImageNet~\citep{imagenet}. However, several additional experiments with other image classification datasets such as MNIST and SVHN can be found in SM, Sections~\ref{appendix:rotated_mnist} and~\ref{appendix:mnist_svhn}.

  We compare NUQ with popular UE methods, which do not require significant modifications to model architectures and training procedures. More specifically, we consider Maximum probability (MaxProb) of softmax response of a NN, entropy of the predictive distribution, the Monte-Carlo (MC) dropout~\citep{Gal2016DropoutAA}, an ensemble of models trained with different random seeds (deep ensemble), the Test-Time Augmentation (TTA; \citep{lyzhov2020greedy}), DDU~\citep{mukhoti2021deterministic}, SNGP~\citep{liu2020simple}, DUQ~\citep{van2020uncertainty}, and an energy-based approach~\citep{liu2020energy}.

  For Monte-Carlo dropout, ensembles, and TTA, we first compute average predicted class probabilities and then compute their entropy (see the ablation study in SM, Section~\ref{appendix:choice_of_uncertainty_measure_for_ensemble}). More details can be found in SM, Section~\ref{sec:architechtures}. For deep ensembles, we fixed the number of models to 5. Additional experiments, where we changed the number of models, are presented in SM, Section~\ref{sec:ensemble:size}.

\subsubsection{CIFAR-100}
\label{sec:cifar100_exps}
  \begin{table*}[t]
    \centering
    \resizebox{\textwidth}{!}{
      \begin{tabular}{|c|c|c|c|c|c|c|c|c|c|c|}
        \hline
        OOD dataset & MaxProb* & Entropy* & Dropout & Ensemble & TTA  & Energy* & DUQ*      & SNGP* &  DDU*    & NUQ*   \\ \hline
        SVHN        & 79.7±1.3 & 81.1±1.6 & 77.6±2.5       & 82.9±0.9  & 81.6±1.2 & 62.0±1.7 & 88.7±6.3 & 86.2±7.4 & \underline{89.6±1.6}          & \textbf{89.7±1.6} \\ \hline
        LSUN        & 81.5±2.0 & 83.0±2.1 & 76.8±5.1       & 86.5±0.8            & 85.0±2.7 & 82.7±0.1 & 90.8±6.7 & 83.7±8.6 & \underline{92.1±0.6} & \textbf{92.3±0.6} \\ \hline
        Smooth     & 76.6±3.5 & 77.8±5.2 & 63.3±3.8       & 83.7±1.2  & 73.2±10.8 & 71.5±4.6 & 91.1±8.4 & 60.9±12.5 & \textbf{97.1±3.1}         & \underline{96.8±3.8}          \\ \hline
      \end{tabular}
    }
    \caption{OOD detection for CIFAR-100 in-distribution dataset with the ResNet-50 neural network. The top two results are shown in bold and underline correspondingly. Evaluation is done for three models trained with different seeds to estimate the standard deviation. Methods requiring a single pass over the data to compute uncertainty estimates are marked with $\protect\mbox{*}$.}
  \label{tab:cifar}
  \end{table*}

  In this experiment, we test UE methods on the out-of-distribution detection task. We treat the OOD detection as binary classification (OOD/not-OOD) using only the uncertainty score. Following the setup from the recent works~\citep{van2020uncertainty,van2021improving,Sastry2020DetectingOE}, we use SVHN, LSUN~\citep{lsun}, and Smooth~\citep{hein2019relu} as OOD datasets. The reported metric is ROC-AUC. 
  
  As a base model, we train ResNet-50 from scratch on the CIFAR-100 dataset. NUQ was applied to the features from the penultimate layer of the model, and the density estimate is given by GMM, as it provides the best results (see the results for other choices of hyperparameters in Section~\ref{sec:ablation}).

  The results are presented in Table~\ref{tab:cifar}. We can clearly see that NUQ and DDU show close results while outperforming the competitors with a significant margin.

\subsubsection{ImageNet}
  \begin{table}
    \centering
    \resizebox{\textwidth}{!}{
    \begin{tabular}{|l|c|c|c|c|c|c|c|c|c|}
      \hline
      OOD dataset & \multicolumn{1}{l|}{MaxProb*} & \multicolumn{1}{l|}{Entropy*} & \multicolumn{1}{l|}{TTA}  & \multicolumn{1}{l|}{Energy*} & \multicolumn{1}{l|}{Ensemble} & DDU* & \multicolumn{1}{l|}{DUQ*} & \multicolumn{1}{l|}{SNGP*} & \multicolumn{1}{l|}{NUQ*} \\
      \hline
      ImageNet-R  & 80.4 & 83.6 & 85.8 & 78.34 & 84.4 & 80.1 & 73.3 & 85.0 & \textbf{99.5} \\
      \hline
      ImageNet-O  & 28.2 & 29.1 & 30.5 & 60.0 & 51.9 & 74.1 & 71.4 & 75.8 & \textbf{82.4} \\
      \hline
    \end{tabular}
    }
    \caption{ROC-AUC score for ImageNet out-of-distribution detection tasks for different methods. Methods requiring a single pass over the data to compute uncertainty estimates are marked with $\protect\mbox{*}$.}
  \label{tab:imagenet}
  \end{table}
  
  \begin{figure*}[t]
    \footnotesize
    \centering
    \begin{minipage}[h]{0.32\linewidth}
    \center{\includegraphics[width=4.2cm]{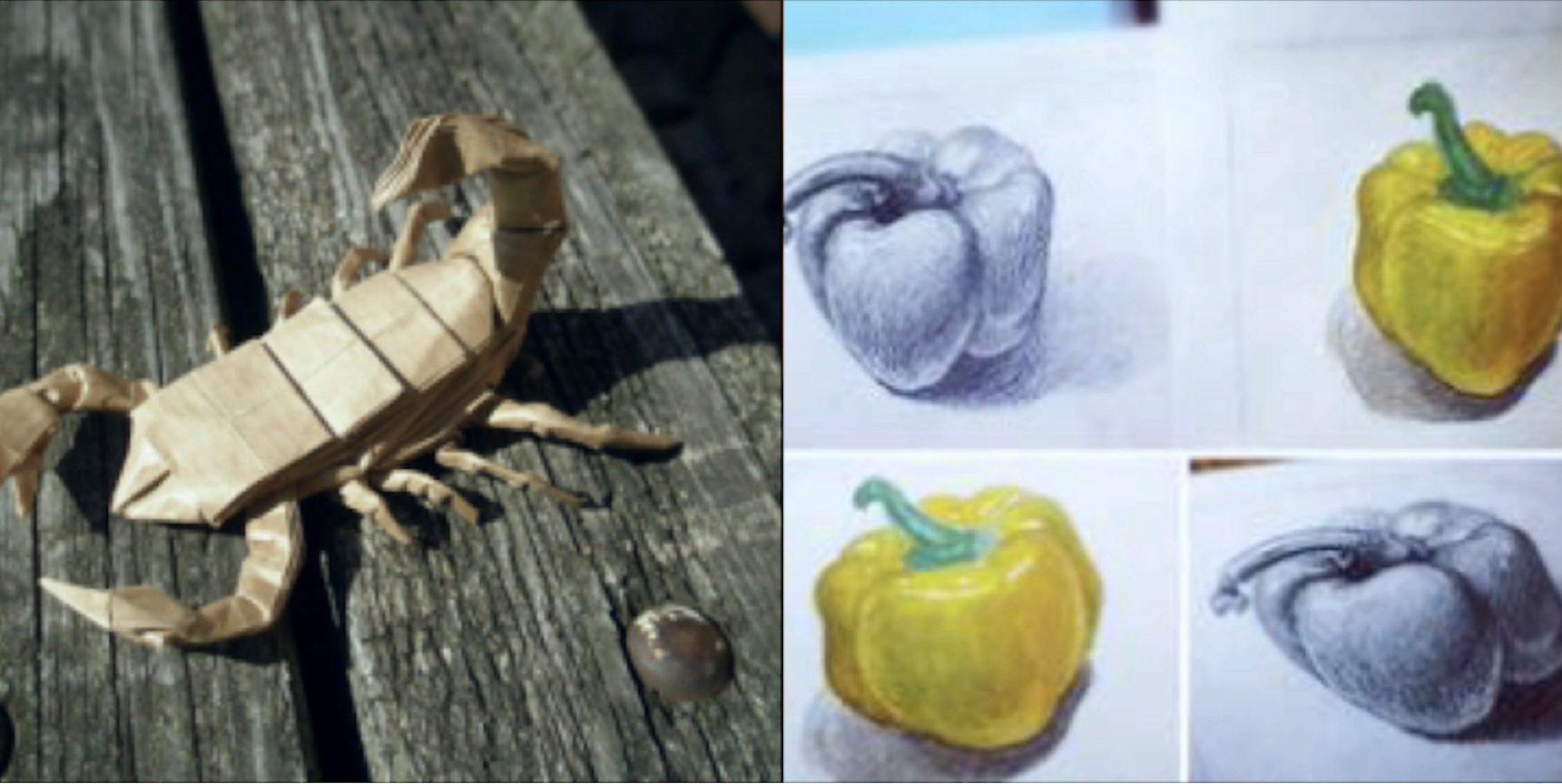} \\
    a) Low uncertainty 
    }
    \end{minipage}~
    \begin{minipage}[h]{0.32\linewidth}
    \center{\includegraphics[width=4.2cm]{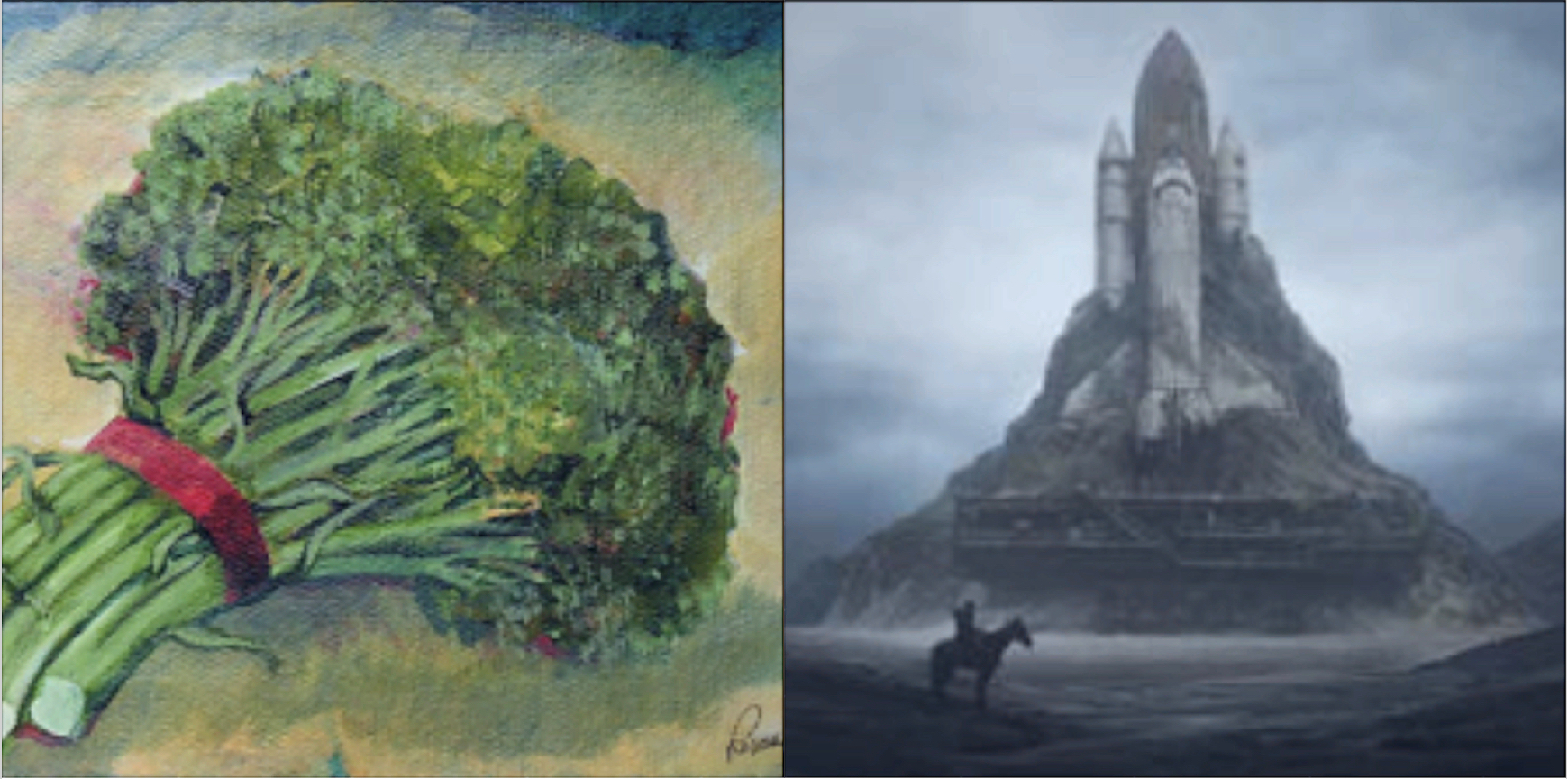} \\
    b) Medium uncertainty 
    }
    \end{minipage}~
    \begin{minipage}[h]{0.32\linewidth}
    \center{\includegraphics[width=4.2cm]{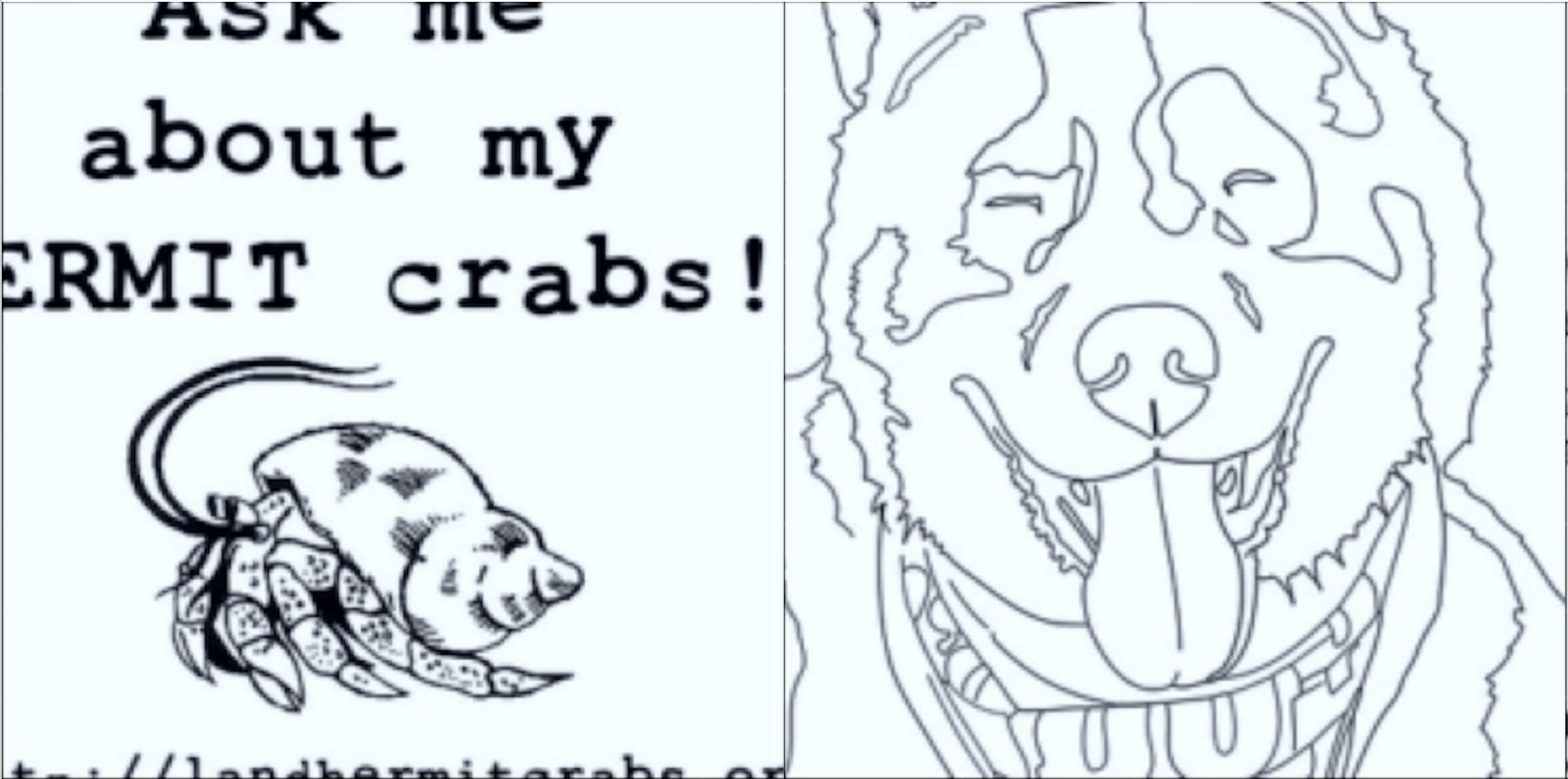} \\
    c) High uncertainty 
    }
    \end{minipage}
    \caption{ Typical OOD images from ImageNet-R ordered from low uncertainty (bottom 10\%) (a) to high uncertainty (top 90\%) (c) by the NUQ \(\hat{\textbf{U}}_e(\vx)\) scores. We can clearly see that low-uncertainty images resemble real-world objects presented in vanilla ImageNet.}
    \label{fig:imagenet_uncertainty_panel}
    \vspace{-0.3cm}
  \end{figure*}

  To demonstrate the applicability of NUQ to large-scale data, we evaluate it in the OOD detection task on ImageNet~\citep{imagenet}. As OOD data, we use the ImageNet-O~\citep{imagenet_o} and ImageNet-R~\citep{imagenet_r} datasets. ImageNet-O consists of images from classes not found in the standard ImageNet dataset. ImageNet-R contains different artistic renditions of ImageNet classes.
  
  In contrast to the previous experiment, we found that for NUQ, it is more beneficial to use KDE as a density estimator $p(\xv)$, rather than GMM. Importantly, it took us approximately 5 minutes to receive uncertainties over all ImageNet datasets with a CPU, i.e. our NUQ implementation is readily applicable to large-scale data (see more details in SM, Section~\ref{appendix:computation_costs}).
  
  The results are summarized in Table~\ref{tab:imagenet}. We see that for ImageNet-O, many methods show good OOD detection quality, but NUQ achieves an almost perfect result. For ImageNet-R, simple approaches completely fail while DDU, SNGP, and DUQ perform well, and NUQ shows the best result with a large margin.

  Note that unlike in the CIFAR-100 experiment, for ImageNet, NUQ significantly outperforms DDU. We conjecture that GMM struggles to approximate density here as an embedding structure is much more complicated for ImageNet compared to CIFAR-100 (see some visualizations in Section~\ref{sec:imagenet_performance}). NUQ is beneficial in this case as KDE is much more flexible than GMM and provides a better result.
  
  Additionally, we looked at some typical samples from ImageNet-R with low, moderate, and high levels of uncertainty as assigned by NUQ, see Figure~\ref{fig:imagenet_uncertainty_panel}. Here, low, medium, and high uncertainties correspond to 10, 50, and 90\% quantiles of the epistemic uncertainty distribution for images from the ImageNet-R dataset. We observe that uncertainty values correspond well to intuitive degree of image complexity compared to the original ImageNet data. Some additional ImageNet experiments are presented in SM, Section~\ref{appendix:additional_imagenet_experiments}.

\subsection{Text Classification}
  Experiments on textual data are performed in low-resource settings, where we train models on small subsamples of original datasets. This regime can be challenging for many UE methods, while NUQ is naturally adapted to it. We compare NUQ to the best performing methods on image classification: DDU and deep ensemble, and to the standard baselines: MC dropout and MaxProb.

  The methods are evaluated on two tasks: OOD detection and classification with a reject option.
  Classification with rejection experiments are conducted on SST-2~\cite{socher-etal-2013-recursive}, MRPC~\cite{dolan-brockett-2005-automatically}, and CoLA~\cite{warstadt-etal-2019-neural}. The evaluation metric is RCC-AUC~\cite{el2010foundations}. Experiments with OOD detection are conducted on ROSTD~\cite{GangalAEG20} and CLINC~\cite{larson-etal-2019-evaluation}, which originally contain instances marked as OOD, and a benchmark composed from SST-2 (in-domain), 20~News Groups~\cite{Lang95}, TREC-10~\cite{li-roth-2002-learning, hovy-etal-2001-toward}, WMT-16~\cite{bojar-etal-2016-findings}, Amazon~\cite{McAuleyL13} (sports and outdoors categories), MNLI~\cite{williams-etal-2018-broad}, and RTE~\cite{rte2005,rte2006,rte2007,rte2009}, where SST-2 is used as an in-domain dataset, while the rest as OOD datasets. The final score is averaged across all OOD datasets. Evaluation metric is ROC-AUC as in the image classification task. The dataset statistics are presented in SM, Section~\ref{sec:textual_hyperparams}.

  We use a pre-trained ELECTRA model with 110 million parameters. The features for NUQ and DDU are taken from the penultimate classification layer. The details of the model, hyperparameter optimization, and UE methods are presented in SM, Section~\ref{sec:textual_hyperparams}.

\subsubsection{Classification with a Reject Option}
  \begin{figure*}[t]
    \footnotesize
    \centering
    \begin{minipage}[h]{0.32\linewidth}
    \center{\includegraphics[width=4.cm]{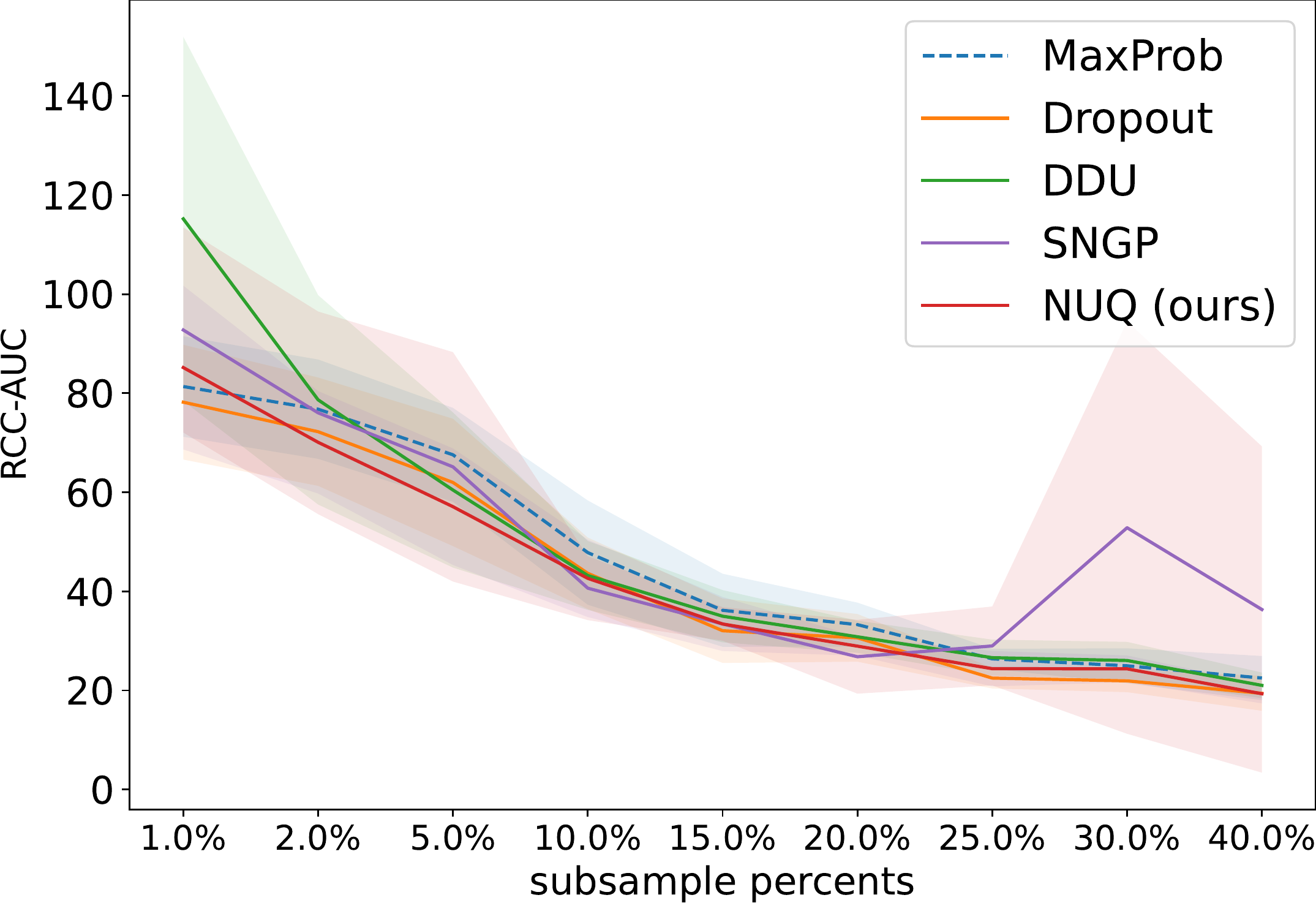} \\
    a) MRPC
    }
    \end{minipage}
    \begin{minipage}[h]{0.32\linewidth}
    \center{\includegraphics[width=4.cm]{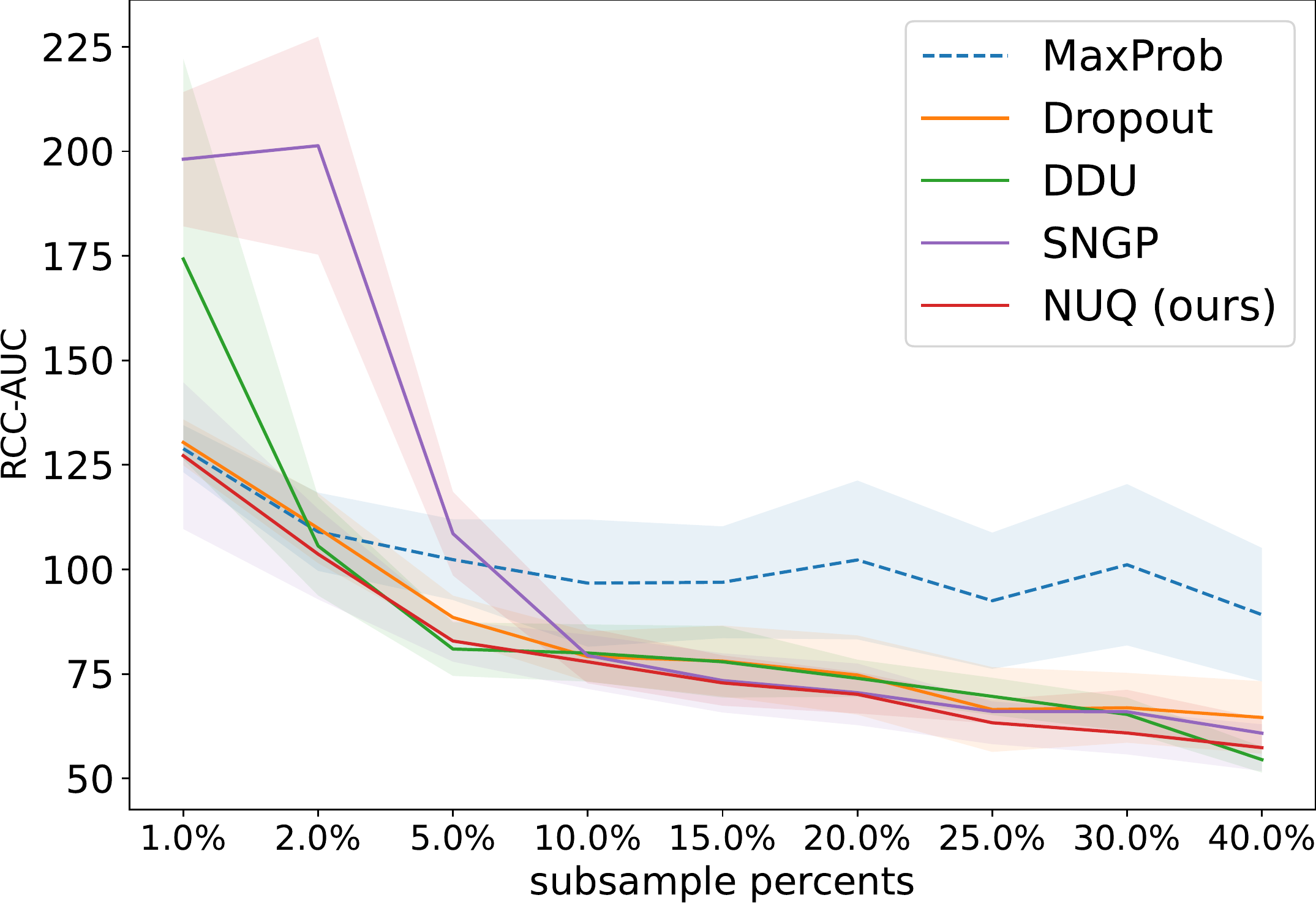} \\
    b) CoLA 
    }
    \end{minipage}
    \begin{minipage}[h]{0.32\linewidth}
    \center{\includegraphics[width=4.cm]{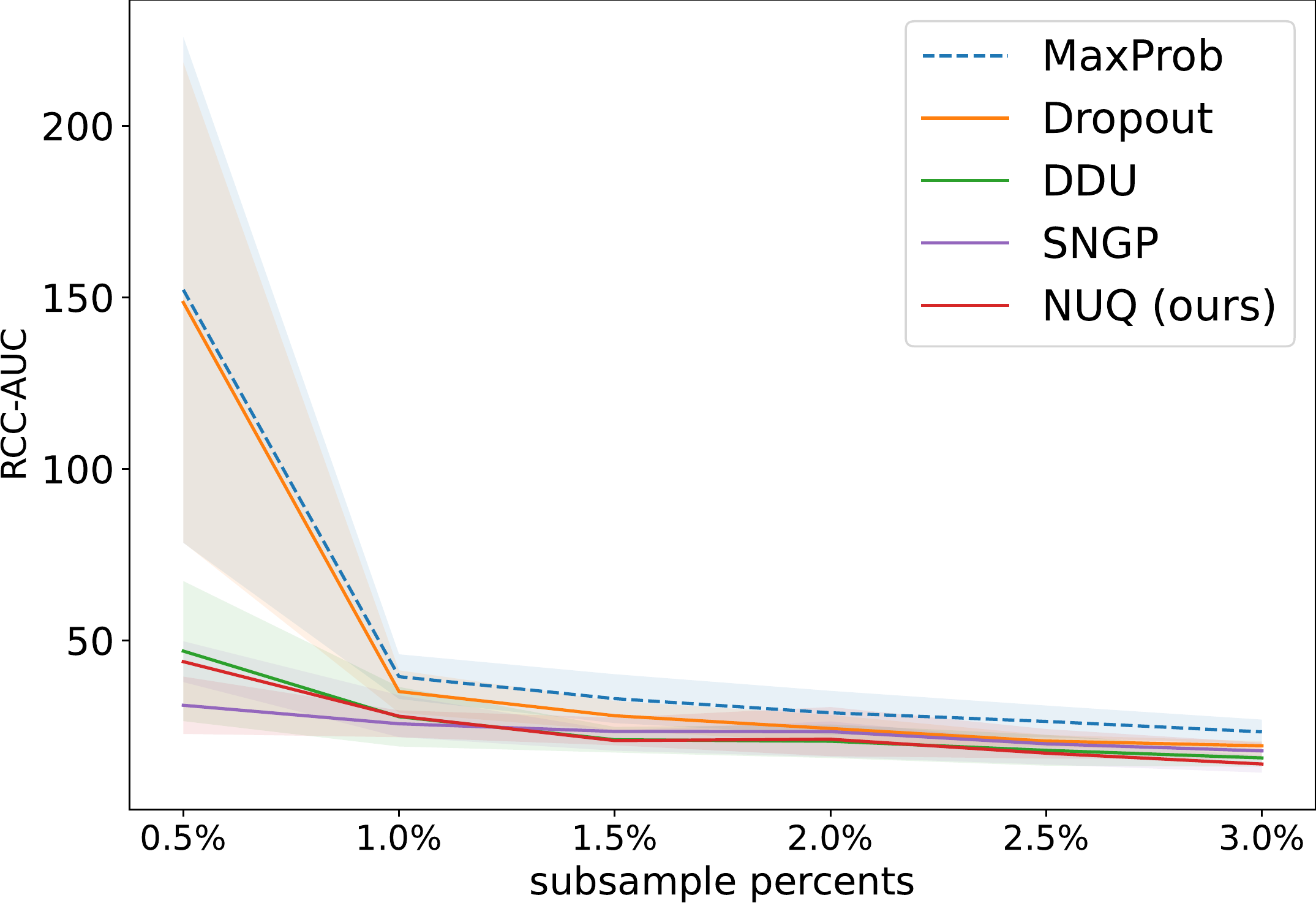} \\
    c) SST-2
    }
    \end{minipage}
    \caption{RCC-AUC$\downarrow$ of classification with rejection depending on  fraction of unlocked training data.}
    \label{fig:miscl_subsamples}
  \end{figure*}

  Figure~\ref{fig:miscl_subsamples} presents the results for classification with rejection. For the MRPC dataset, we can note that NUQ, SNGP, and MC dropout show similar results to the MaxProb baseline. DDU stands out, demonstrating substantially worse performance in settings with small amount of training data.

  A similar dynamics for DDU can be noted on the CoLA dataset, where it works substantially worse than the baseline when only 1\% of the training data is available. SNGP in this experiment manages to reach other methods only when 10\% of training data is unlocked. On contrary, NUQ is always on par or better than the baseline outperforming all other computationally efficient methods and has similar performance as computationally expensive MC dropout. 

  On SST-2, NUQ works similar to SNGP and DDU, substantially outperforming the MaxProb baseline, and MC dropout in the extremely low resource setting. Starting from 1.5\%, all methods work similarly with small advantage over the baseline.

  Overall, we can conclude that in the settings with a small amount of training data, NUQ can be the best choice for estimating uncertainty for the classification with a reject option: it works similar to other methods when there is much training data and does not deteriorate in the low-resource regime, demonstrating much better results than others. SNGP is able to approach NUQ on SST-2. However, its performance is not stable, which is illustrated by poor results on CoLA. DDU sometimes fails to outperform the baseline with small amount of training data, which might be due to its reliance on the assumption that training data has a Gaussian distribution, which does not hold in this setting.

\subsubsection{Out of Distribution Detection} 
  \begin{figure*}[t]
    \footnotesize
    \centering
    \begin{minipage}[h]{0.32\linewidth}
    \center{\includegraphics[width=4.cm]{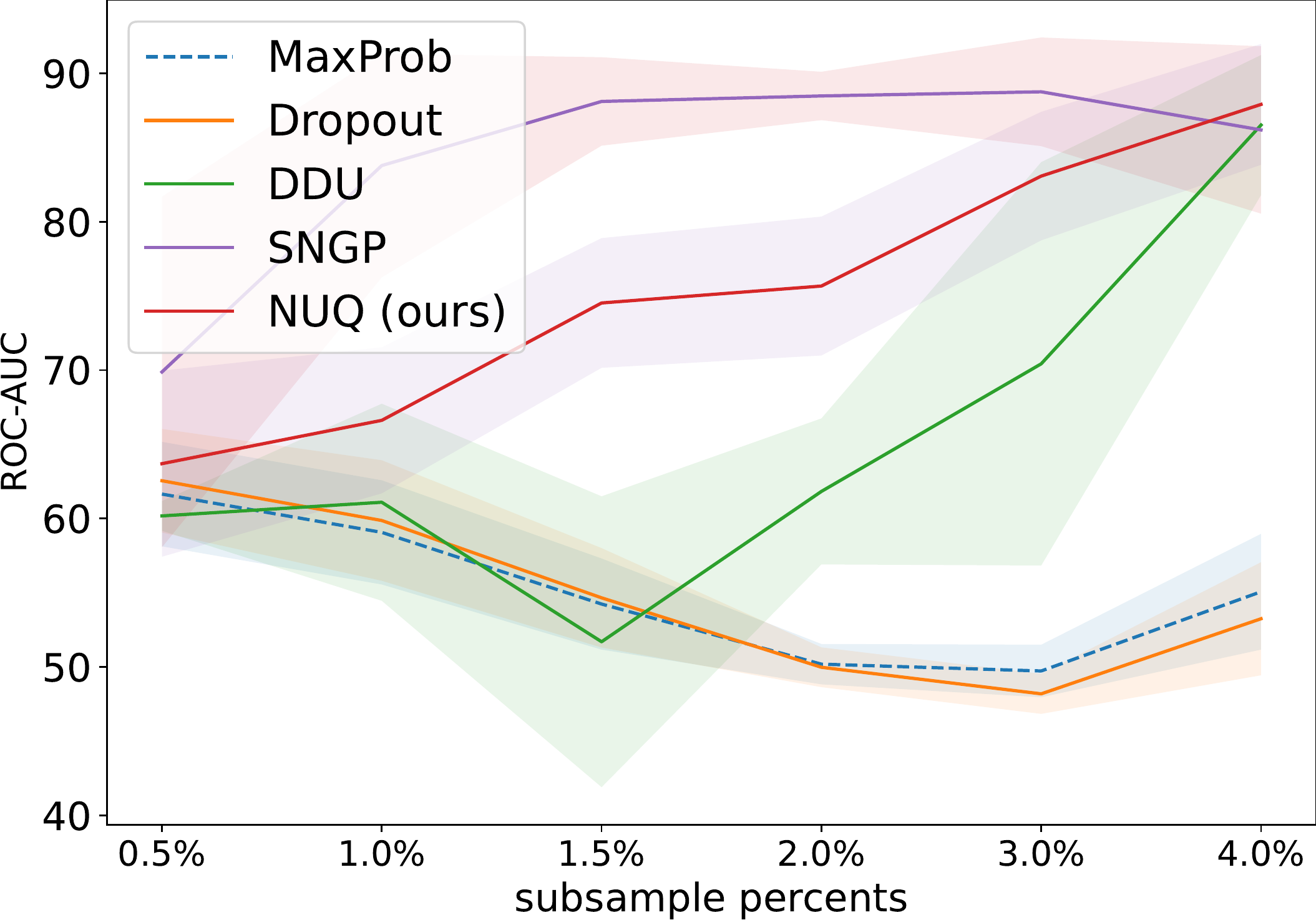}\\
    a) ROSTD
    }
    
    \end{minipage}
     \begin{minipage}[h]{0.32\linewidth}
    \center{\includegraphics[width=4.cm]{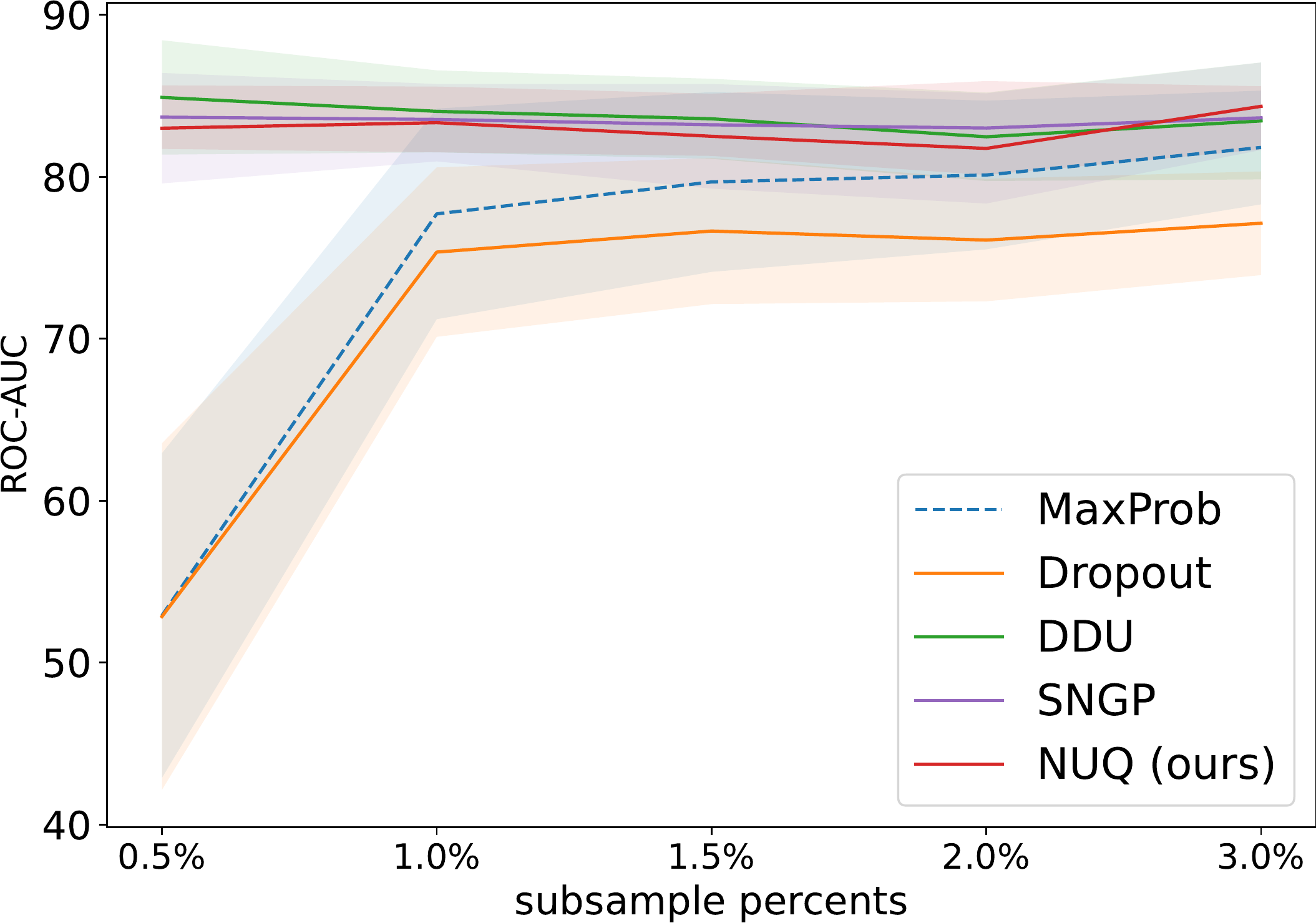}\\
    b) SST-2
    }
    
    \end{minipage}
    \begin{minipage}[h]{0.32\linewidth}
    \center{\includegraphics[width=4.cm]{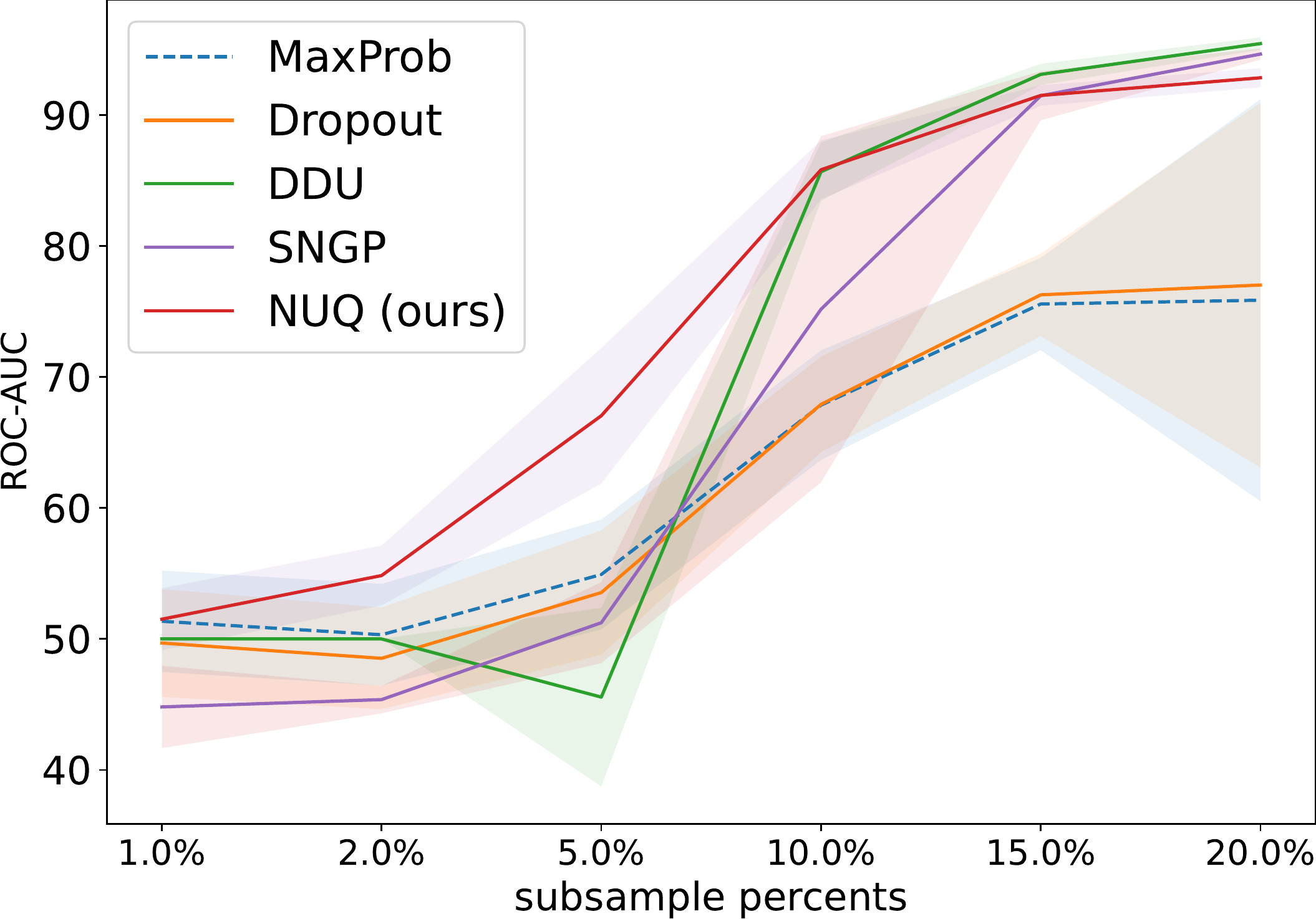}\\
    c) CLINC 
    }

    \end{minipage}
    \caption{ROC-AUC$\uparrow$ of OOD detection depending on the fraction of unlocked training data.}
    \label{fig:ood_subsamples}
    
    \vspace{-0.5cm}
  \end{figure*}

  Figure~\ref{fig:ood_subsamples} presents the results of OOD detection. We see that NUQ confidently outperforms MC dropout on all datasets and outperforms DDU on ROSTD and CLINC. This is  especially notable for extremely low-resource settings. SNGP has some advantage over NUQ on ROSTD, but it substantially falls behind on CLINC. Moreover, unlike SNGP, NUQ always outperforms the MaxProb baseline, therefore, it might be a better choice for OOD detection in low-resource regimes. Finally, for ROSTD and CLINC datasets, DDU works poorly, sometimes failing to outperform the MaxProb baseline, which stems from the incorrect assumption about the Gaussian distribution of training data.


\section{Conclusions}
\label{sec:conclusions}
  This work proposes NUQ, a new principled uncertainty estimation method that applies to a wide range of neural network models. It does not require retraining the model and acts as a postprocessing step working in the embedding space induced by the neural network. NUQ significantly outperforms the competing approaches with only the recently proposed DDU method~\citep{mukhoti2021deterministic} showing comparable results. Importantly, in the most practical example of OOD detection for ImageNet data, NUQ shows the best results with a significant margin. NUQ is also superior to DDU and other methods on text classification datasets in both OOD detection and classification with rejection. The code to reproduce the experiments is available online at~\url{https://github.com/stat-ml/NUQ}. 

  We hope that our work opens a new perspective on uncertainty quantification methods for deterministic neural networks. We also believe that NUQ is suitable for in-depth theoretical investigation, which we defer to future work.

  \textbf{Acknowledgements.} The research was supported by the Russian Science Foundation grant 20-71-10135

\bibliographystyle{plain}
\bibliography{bibliography}

\section*{Checklist}

The checklist follows the references.  Please
read the checklist guidelines carefully for information on how to answer these
questions.  For each question, change the default \answerTODO{} to \answerYes{},
\answerNo{}, or \answerNA{}.  You are strongly encouraged to include a {\bf
justification to your answer}, either by referencing the appropriate section of
your paper or providing a brief inline description.  For example:
\begin{itemize}
  \item Did you include the license to the code and datasets? \answerYes{See Section ...}
  \item Did you include the license to the code and datasets? \answerNo{The code and the data are proprietary.}
  \item Did you include the license to the code and datasets? \answerNA{}
\end{itemize}
Please do not modify the questions and only use the provided macros for your
answers.  Note that the Checklist section does not count towards the page
limit.  In your paper, please delete this instructions block and only keep the
Checklist section heading above along with the questions/answers below.

\begin{enumerate}

\item For all authors...
\begin{enumerate}
  \item Do the main claims made in the abstract and introduction accurately reflect the paper's contributions and scope?
    \answerYes{}
  \item Did you describe the limitations of your work?
    \answerYes{}
  \item Did you discuss any potential negative societal impacts of your work?
    \answerNo{Our proposed method is widely applicable, almost any modern neural network can be equipped with this mechanism. Potential social impact will highly depend on the specific domain and task.}
  \item Have you read the ethics review guidelines and ensured that your paper conforms to them?
    \answerYes{}
\end{enumerate}

\item If you are including theoretical results...
\begin{enumerate}
  \item Did you state the full set of assumptions of all theoretical results?
    \answerYes{}
        \item Did you include complete proofs of all theoretical results?
    \answerYes{Proofs for our results are in Supplementary Material, see Section~\ref{sec:consistency_proof}}
\end{enumerate}

\item If you ran experiments...
\begin{enumerate}
  \item Did you include the code, data, and instructions needed to reproduce the main experimental results (either in the supplemental material or as a URL)?
    \answerYes{We provide the link to our code in Conclusion.}
  \item Did you specify all the training details (e.g., data splits, hyperparameters, how they were chosen)?
    \answerYes{We discuss this in Supplementary Material, Sections \ref{sec:architechtures} and \ref{sec:textual_hyperparams}}
        \item Did you report error bars (e.g., with respect to the random seed after running experiments multiple times)?
    \answerYes{}
        \item Did you include the total amount of compute and the type of resources used (e.g., type of GPUs, internal cluster, or cloud provider)?
    \answerNo{It is not included in the main paper, but will be addressed in Supplementary material.}
\end{enumerate}

\item If you are using existing assets (e.g., code, data, models) or curating/releasing new assets...
\begin{enumerate}
  \item If your work uses existing assets, did you cite the creators?
    \answerYes{We provide references to the datasets in the \ref{sec:experiments} section. For our comparisons with existing approaches we have used the author's implementations that were provided in the respective original papers that we cite.}
  \item Did you mention the license of the assets?
    \answerNo{We use only open source code and datasets.} 
  \item Did you include any new assets either in the supplemental material or as a URL?
    \answerYes{We provide the link to our code in Conclusion.}
  \item Did you discuss whether and how consent was obtained from people whose data you're using/curating?
    \answerNA{All the assets used are free for research purposes.}
  \item Did you discuss whether the data you are using/curating contains personally identifiable information or   offensive content?
    \answerNo{We use publicly available datasets to benchmark our results and show only numeric values and abstract plots. To the best of our knowledge we have not exposed any personal information or offensive content.}
\end{enumerate}

\item If you used crowdsourcing or conducted research with human subjects...
\begin{enumerate}
  \item Did you include the full text of instructions given to participants and screenshots, if applicable?
    \answerNA{}
  \item Did you describe any potential participant risks, with links to Institutional Review Board (IRB) approvals, if applicable?
    \answerNA{}
  \item Did you include the estimated hourly wage paid to participants and the total amount spent on participant compensation?
    \answerNA{}
\end{enumerate}

\end{enumerate}

\newpage
\appendix


\section{Multiclass Generalization for Uncertainties}
\label{sec:multiclass_generalization}
  In this section we show, how our method can be generalized from binary classification to multiclass problems. Consider data pairs $(X, Y) \sim \PP$. Now, $X \in \mathbb{R}^d$ and $Y \in \{1, \dots, C\}$, where $C$ is the number of classes. We also denote $\eta_c(\xv) = \PP(Y = c \mid X = \xv)$.

  Let us start with the Bayes risk: 
  \begin{EQA}[c]
    \PP(Y\neq g^*(X) \mid X = \xv) = 1 - \PP(Y = g^*(X) \mid X = \xv)
    = 1 - \max_c \eta_c(\xv) = \min_c \bigl\{1 - \eta_c(\xv)\bigr\},
  \end{EQA}
  where $g^*(\xv):= \text{arg}\max_c \eta_c(\xv)$ is the Bayes optimal classifier.

  Let us further move to the excess risk and denote by $\hat{\eta}_c(\xv)$ some estimator of conditional probability. Analogously, $g(\xv):= \text{arg}\max_c \hat{\eta}_c(\xv)$ and we can bound the excess risk in the following way:
  \begin{EQA}
    && \PP(Y \neq g(X) \mid X = \xv) - \PP(Y \neq g^*(X) \mid X = \xv) = \eta_{g^*(\xv)}(\xv) - \eta_{g(\xv)}(\xv)
    \\ 
    &=& \eta_{g^*(\xv)}(\xv) - \hat{\eta}_{g^*(\xv)}(\xv) + \hat{\eta}_{g^*(\xv)}(\xv) -\hat{\eta}_{g(\xv)}(\xv) + \hat{\eta}_{g(\xv)}(\xv) - \eta_{g(\xv)}(\xv)
    \\
    &\leq& \bigl|\eta_{g^*(\xv)}(\xv) - \hat{\eta}_{g^*(\xv)}(\xv)\bigr| + \bigl|\eta_{g(\xv)}(\xv) - \hat{\eta}_{g(\xv)}(\xv)\bigr|,
  \end{EQA}
  where we used the fact that $\hat{\eta}_{g^*(\xv)}(\xv) - \hat{\eta}_{g(\xv)}(\xv) \leq 0$ for any $\xv$.

  The expectation of the right hand side in the case of kernel density estimator can be upper bounded by $2\sqrt{\frac{2}{\pi}}\tau(\xv)$, where 
  \begin{EQA}[c]
    \tau^2(\xv) = \frac{1}{N} \frac{\max_c \bigl\{\sigma_c^2(\xv)\bigr\}}{p(\xv)} \int \bigl[K_h(\uv)\bigr]^2d\uv
  \end{EQA}
  and $\sigma_c^2(\xv) = \eta_c(\xv) \bigl(1 - \eta_c(\xv)\bigr)$. Total uncertainty for multiclass problem is thus
  \begin{EQA}[c]
    \textbf{U}_t(\xv) = \min_c \bigl\{1 - \eta_c(\xv)\bigr\} + 2 \sqrt{\frac{2}{\pi}} \tau(\xv).
  \label{eq:total_un_def_multi}
  \end{EQA}

\section{Architecture and Training Details for Image Datasets}
\label{sec:architechtures}  

\paragraph{Base Model.}
  For CIFAR-100 and ImageNet-like datasets, we are using ResNet50 as a base model, with or without spectral normalization~\citep{Miyato2018SpectralNF}. For the spectral normalization, we use 3 iterations of the power method. We use a ResNet50 architecture with implementation from PyTorch~\citep{NEURIPS2019_9015}. This architecture was implemented for the ImageNet dataset; thus, for the CIFAR-100, we had to adapt it. We changed the first convolutional layer and used kernel size 3x3 with stride 1 and padding 1 (instead of kernel size 7x7 with stride 2 and padding 3 used for ImageNet). For CIFAR-100, we train the model for 200 epochs with an SGD optimizer, starting with a learning rate of 0.1 and decaying it 5 times on 60, 120, and 160 epoch. For ImageNet, we train the model for 90 epochs with an SGD optimizer and learning rate decaying 10 times every 30 epochs.
  
  For MNIST, we train a small convolutional neural network with three convolutional layers with padding of 1 and kernel size of 3. Each of these layers is followed by a batch normalization layer. Finally, it has a linear layer with Softmax activation. This network achieves an accuracy of 0.99 on the holdout set. 
  
  We refer readers to our code for more specific details.
  
\paragraph{Ensemble.}
  For ensemble, we use a combination of 5 base models trained with different random seeds.

\paragraph{Test-Time Augmentation (TTA).}
  For TTA, we use a base model and apply different transformations on the inference stage. Images of CIFAR-100 are randomly cropped with padding 4, randomly horizontally flipped, and randomly rotated up to 15 degrees. ImageNet is randomly cropped from 256 to 224, randomly horizontally flipped, and the color was jittered (0.02).

\paragraph{Spectrally Normalized Models.}
  For SNGP, DDU and NUQ, we need spectral normalized models to extract features. We wrapped each convolutional and linear layer with spectral normalization (PyTorch implementation). We used 3 iterations of the power method in our experiments.

\section{Hyperparameters Selection Strategies}
\label{appendix:hyperparams}
  In this section, we collect all the hyperparameters one should choose to use NUQ, as well as we discuss strategies for choosing them.

\subsection{Bandwidth}
\label{appendix:hyperparams_bandwidth}
  Probably the most important hyperparameter is bandwidth for kernels used in Nadaraya-Watson estimator. It is extensively studied in literature (see, for example, \cite{liao2010improving,turlach1993bandwidth} and references therein). However, there is still no silver bullet, and often the best approach for bandwidth selection is specific to a particular applied problem.

  In this work we use cross-validation to estimate the accuracy of resulting Nadaraya-Watson based classifier, and tune the bandwidth to maximize it. We implicitly use an assumption, that the bandwidth, which is good for in-distribution classification, will be still good for out-of-distribution detection. Our experiments show, that this simple idea works well in practice.

\subsection{Number of Nearest Neighbors}
\label{appendix:hyperparams_nn}
  In this section, we demonstrate, how the number of nearest neighbors affects the results. For the sake of demonstration, we have chosen the most challenging dataset (ImageNet) and its out-of-distribution counterparts. It is important to note, that to make the experiment faster, we used only a subsample (75k) of the ImageNet dataset.

  In Figure~\ref{fig:ablation_nn} we can see, that this hyperparameter has a wide interval, within which values do not affect the performance of the overall method. Already starting from 20 nearest neighbours the results are close to the optimal ones (both in terms of accuracy and out-of-distribution ROC AUC). This is due to the fact that typical kernels decay fast with distance, and there is no need to take a lot of nearest neighbors into account.

  \begin{figure}
    \centering
    \includegraphics[width=\textwidth]{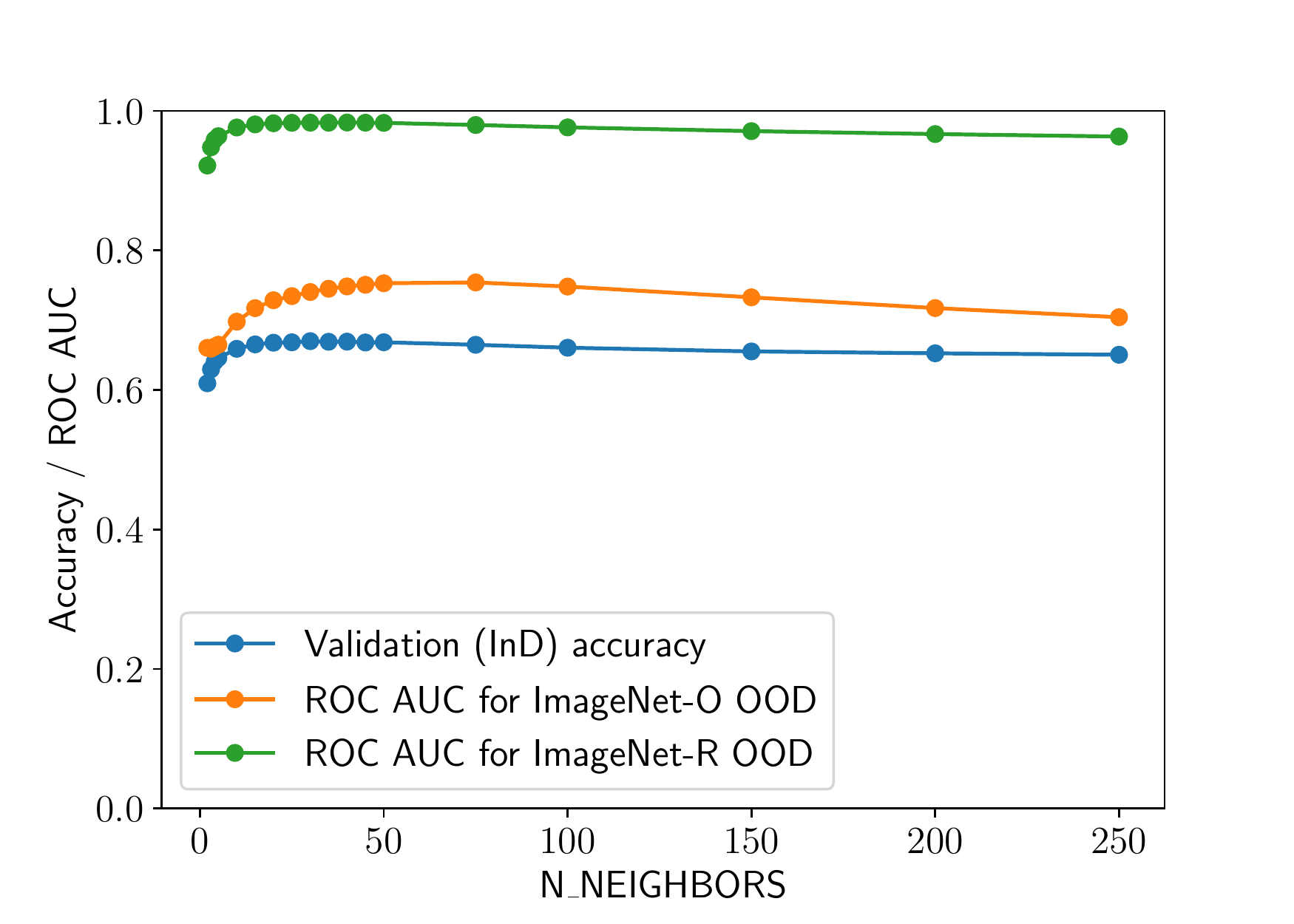}
    \caption{The dependence of ROC AUC for OOD detection and classification accuracy on the cross-validated training dataset as a function of the number of nearest neighbours used for approximation.}
    \label{fig:ablation_nn}
  \end{figure}

\subsection{Marginal Density Estimation}
\label{appendix:hyperparams_kde}
  To compute epistemic uncertainty according to equation~\eqref{gaussian_asymptotic}, we need to estimate $p(\xv)$. In general, there are multiple possible ways how one can do it (including Variational Autoencoders~\cite{kingma2013auto} and Normalizing Flows~\cite{papamakarios2021normalizing}).
  However, for computational efficiency something lightweight is required. Specifically, we propose two ways. The first one is to use another KDE with the same kernel and bandwidth used for NW estimator. Another approach is to use training dataset embeddings and corresponding labels to assign a Gaussian distribution per class (GMM). Note, that it is not a Gaussian Mixture Model in the traditional sense, typically fitted with EM-algorithm. From our experiments (see Table~\ref{tab:ablation_density}) we see, that both of them perform on par, with some advantages of GMM for CIFAR-100 dataset, while for ImageNet KDE works better.

\subsection{Kernel for KDE}
\label{appendix:hyperparams_kernels}
  In principle, one could choose different kernels for both, KDE (see Table~\ref{tab:ablation_density}) and NW-regression. But typically, as kernels decay fast with the distance, it does not affect much on the result. Thus, Gaussian (RBF) kernel is a good default choice.
  
  Below, we study the choice of a kernel to plug-in in our approach. First, let us rewrite $K_h(\uv)$ as follows:
  \begin{EQA}[c]
    K_h(\uv) = \prod_{i=1}^d K\Bigl(\frac{u_i}{h}\Bigr) = \prod_{i=1}^d K(z_i).
  \end{EQA}
  
  We consider the different choices of kernels, see Table~\ref{tab:kernels}.
  \begin{table}[ht!]
    \centering
    \begin{tabular}{c|c|c}
      \textbf{Kernel name}  & \textbf{Formula $K(z)$} & \textbf{Integral} $\int K_h(\uv)^2d\uv$ \\
      \hline 
      Gaussian (RBF) & $\frac{1}{\sqrt{2\pi}}\exp{\bigl\{-z^2\bigr\}}$ & $\frac{h^d}{2\sqrt{\pi}}$ \\ 
      \hline 
      Sigmoid & $\frac{2}{\pi} \frac{1}{\exp\{-z\} + \exp\{z\}}$ & $\frac{2h^d}{\pi^2}$ \\
      \hline
      Logistic & $\frac{1}{\exp\{-z\} + 2 + \exp\{z\}}$ & $\frac{h^d}{6}$ \\
      \hline
    \end{tabular}
    \caption{Different types of kernels $K(z)$ considered and corresponding values of the integral $\int K_h(\uv)^2d\uv$.}
  \label{tab:kernels}
  \end{table}

\section{Consistency of NUQ-based Classification with the Reject Option}
\label{sec:consistency_proof}
  \def\vt{\mathbf{t}}
  \def\vy{\mathbf{y}}
  
  In this section, we derive a non-asymptotic upper bound of the excess risk used to obtain the consistency result in Section~\ref{sec:consistency}. First, using results of the previous section, notice, that for an arbitrary rejection rule $\alpha(\xv)$ the excess risk of $\risk_{\lambda}(\xv)$ is at most
  \begin{align*}
    \EE_{\dataset} \left \{ \risk_{\lambda}(\xv) - \risk_{\lambda}^*(\xv) \right \} \leqslant 2 \EE_{\dataset} \left \{ \max_{c \in C} |\eta_c(\xv) - \hat \eta_c(\xv)| \indicator\{\alpha(\xv) = 0\} \right \} + |\lambda - \risk^*(\xv)| \PP_{\dataset} \left (\alpha(\xv) \neq \alpha^*(\xv) \right ).
  \end{align*}
  As previously,
  \begin{align*}
    \frac{\eta_c(\xv) - \hat \eta_c(\xv)}{\tau_c(\xv)} \to \normalDistribution(0, 1), \quad \tau_c(\xv) = \Vert K \Vert_2 \sqrt{\frac{\eta_c(\xv)(1 - \eta_c(\xv)}{N h^d p(\xv)}}.
  \end{align*}
  Thus, asymptotically
  \begin{align*}
    \PP \left (
        \eta_c(\xv) - \hat \eta_c(\xv) \leqslant z_{\beta} \Vert K \Vert_2 \sqrt{\frac{\eta_c(\xv) (1 - \eta_c(\xv))}{N h^d \hat p(\xv)}}
    \right ) \leqslant \beta.
  \end{align*}
  Consequently,
  \begin{align*}
    \PP \left (
        \min_c \{1 - \hat \eta_c(\xv)\} \leqslant \lambda - \max_c z_{1 - \beta} \Vert K \Vert_2 \sqrt{\frac{\hat \eta_c(\xv) (1 - \hat \eta_c(\xv))}{N h^d \hat p(\xv)}}
    \right ) \leqslant \beta |C|.
  \end{align*}

  That leads us to the procedure described as Algorithm~\ref{algo: test classification}.
  \begin{algorithm}
        
        \KwIn{Samples $(X_i, Y_i)$, bandwidth $h$, parameters $\lambda, \beta$}
        \KwOut{Accept or reject the regression result}
        
        \BlankLine
        
        Calculate $\hat{p}(\xv), \hat{\eta}_c(\xv)$ for $c \in C$;

        \eIf{ the density estimation $\hat{p}(\xv) > 0$ and
            the criterion holds:
            \begin{align*}
                \min_c(1 - \hat \eta_c(\xv)) \leqslant \lambda - \max_c z_{1 - \beta / |\classes|} \Vert K \Vert_2 \sqrt{\frac{\hat \eta_c (\xv) (1 - \hat \eta_c (\xv))}{N h^d \hat p(\xv)}} 
            \end{align*}
        }
        { Accept results of the regression }
        { Reject}
        \caption{Acceptance testing for classification}
        \label{algo: test classification}
  \end{algorithm}
      
  We formulate a number of mild assumptions:
  \begin{assumption}
  \label{assumption: regression funciton}
    There exist Hessians of functions $\eta_c(\xv)$, $c \in \classes$, and their spectral norms are bounded by some constant $M_\eta$.
  \end{assumption}

  \begin{assumption}
  \label{assumption: density}
    $L_2$-norms of $\nabla p$ and $\nabla^2 p$ are bounded by constants $L_d$ and $M_d$ respectively, i.e.
    \begin{align*}
        \Vert \nabla p(\xv) \Vert_2 = \sqrt{\sum_i \bigl(\nabla_i p(\xv)\bigr)^2} \leqslant L_d, \quad \sup_{\xv: \Vert \xv \Vert_2 = 1} \Vert \xv^T \nabla^2 p(\xv) \Vert_2 \leqslant M_d.
    \end{align*}
  \end{assumption}

  \begin{assumption}
  \label{assumption: kernel}
    There exist finite values
        \begin{align*}
            \max_\vt K(\vt),  \quad \int_{\RR^d} \Vert \vt \Vert_2 K^2(\vt) d\vt, \quad \int_{\RR^d} K^2(\vt) d\vt,
            \quad \int_{\RR^d} \Vert \vt \Vert_2^2 K(\vt) d\vt,
        \end{align*}
        while 
        \begin{align*}
            \int_{\RR^d} \vt K(\vt) d\vt = 0.
        \end{align*}
  \end{assumption}

    Under these assumptions, we state the following theorem:
  \begin{theorem}
  \label{theorem: main theorem}
    Suppose that assumptions~\ref{assumption: regression funciton}-\ref{assumption: kernel} hold, $p(\xv) > 0$, and $\beta < 1/2$. Define $\Delta = |\lambda - \risk^*(\xv)|$. Then, if $\lambda < \risk^*(\xv)$
    \begin{align*}
        \EE_{\dataset} \left \{ \risk_{\lambda}(\xv) - \risk_{\lambda}^*(\xv) \right \}
        \leqslant |\classes| A_\Delta,
    \end{align*}
    where
    \begin{align*}
        A_\Delta & = 2 \left \{
            R  \wedge (\indicator\{\Delta \leqslant h^2 \kappa_\Delta \} \vee q_\Delta
        ) \right \} + \Delta (\indicator\{\Delta \leqslant h^2 \kappa_\Delta \} \vee q_\Delta), \\ 
        R & =  2 h^2 \kappa_\Delta + 2 \sqrt{
            \frac{\pi \left \{12 \left ( \Vert K \Vert_2^2 + \frac{h L_d}{p(\xv)} \int_{\RR^d} \Vert \vt \Vert_2 K^2(\vt) d\vt \right ) + \max_\vt K(\vt) \right \}}{ N h^d p(\xv)}
        }, \\
        q_\Delta & = \exp \left (
            - \frac{1}{2} \frac{N h^d p(\xv) (\Delta - h^2 \kappa_\Delta)^2}{(1 + \Delta)^2 \left (
                \Vert K \Vert_2^2
                +
                \frac{h L_d}{p(\xv)} \int_{\RR^d} \Vert \vt \Vert K^2(\vt) d\vt
            \right )
            + \frac{1}{3} (\Delta - h^2 \kappa_\Delta) \max_\vt K(\vt)
            }
        \right ),  
    \end{align*}
    and
    \begin{align*}
        \kappa_\Delta & = \frac{1}{p(\xv)} \left (
            M_d + 2 d M_\eta L_d + 2 d L_d \sqrt{M_\eta}
        \right ) \int_{\RR^d} \Vert t \Vert_2^2 K(\vt) d\vt.
    \end{align*}
    If $\risk^*(\xv) \leqslant \Delta$, 
    \begin{align*}
        \EE_{\dataset} \risk(\xv) - \risk^*(\xv)
        \leqslant
        |\classes| \left (R + \indicator\left \{\Delta \leqslant h^2 \kappa_{\Delta} + z_{1 - \beta/|\classes|} \Vert K \Vert_2 \sqrt{\frac{1}{2 N h^2 p(\xv)}}
        \right \} \vee \tilde q_\Delta
        + \indicator \left \{1/2 \leqslant h^2 \kappa_p \right \} \vee q_p \right ).
    \end{align*}
    Here $\tilde q_\Delta$ differs from $q_\Delta$ by replacing $h^2 \kappa_{\Delta}$ with $h^2 \kappa_\Delta + z_{1 - \beta/|\classes|} \Vert K \Vert_2 \sqrt{\frac{1}{2 N h^2 p(\xv)}}$, while
    \begin{align*}
        q_p & = \exp \left (
            - \frac{\frac 1 2 N h^2 p(\xv) (1/2 - h^2 \kappa_p)^2}{\Vert K \Vert_2^2 + \frac{h L_d}{p(\xv)} \int_{\RR^d} \Vert \vt \Vert_2 K^2(\vt) + \max_\vt K(\vt) (1/2 - h^2 \kappa_p)}
        \right ), \\
        \kappa_p & = \frac{M_d}{2 p(\xv)} \int_{\RR^d} \Vert \vt \Vert^2 K(\vt) d\vt.
    \end{align*}
  \end{theorem}

  Notice, that Theorem~\ref{theorem: asymptotic consistency} follows from Theorem~\ref{theorem: main theorem} as all the terms in the upper bound tend to zero when $h \to 0$ and $\nsize h^d \to \infty$ with $\nsize \to \infty$.

  \begin{proof}[Proof of Theorem~\ref{theorem: main theorem}]
    For a reminder, the excess risk is
    \begin{align*}
        \EE_{\dataset} \left \{ \risk_{\lambda}(\xv) - \risk_{\lambda}^*(\xv) \right \} \leqslant 2 \EE_{\dataset} \left \{ \max_{c \in \classes} |\eta_c(\xv) - \hat \eta_c(\xv)| \indicator\{\alpha(\xv) = 0\} \right \} + \Delta \cdot \PP_{\dataset} \left (\alpha(\xv) \neq \alpha^*(\xv) \right ).
    \end{align*}
    First, rewrite the expectation as an integral
    \begin{align*}
        \EE |\hat \eta_c(\xv) - \eta_c(\xv)| \indicator \{\alpha(\xv) = 0\}
        & =
        \int_{0}^{+ \infty}
        \PP (|\hat \eta_c(\xv) - \eta_c(\xv)| \indicator \{\alpha(\xv) = 0\} \geqslant t) dt \\
        & = 
        \int_{0}^{+ \infty}
        \PP \left (
            |\hat \eta_c(\xv) - \eta_c(\xv)| \geqslant t \text{ and } \alpha(\xv) = 0
        \right ) dt \\
        & \leqslant
        \int_{0}^{+\infty}
            \PP \left (
                |\hat \eta_c(\xv) - \eta_c(\xv)| \geqslant t \text{ and } \sum_i K_h(X_i - \xv) \neq 0
            \right ) dt \\
        & \leqslant
        \int_0^1
            \PP \left (
                |\hat \eta_c(\xv) - \eta_c(\xv)| \geqslant t \text{ and } \sum_i K_h(X_i - \xv) \neq 0
            \right ) dt,
    \end{align*}
    since we abstain if $\hat p(\xv) = 0$. Due to Lemma~\ref{lemma: Bernstein bounding},
    \begin{align*}
        & \int_0^1
            \PP \left (
                |\hat \eta_c(\xv) - \eta_c(\xv)| \geqslant t \text{ and } \sum_i K_h(X_i - \xv) \neq 0
            \right ) dt
        \leqslant \\
        & \leqslant 2 h^2 \kappa + 2 \sqrt{
            \frac{\pi \left \{\left ( \Vert K \Vert_2^2 + \frac{h L_d}{p(\xv)} \int_{\RR^d} \Vert \vt \Vert_2 K^2(\vt) d \vt \right ) + \max_\vt K(\vt) \right \}}{ 2 N h^d p(\xv)}
        }.
    \end{align*}
    Here we use the Poisson integral. Denote this upper bound by $R$.

    \textbf{Now, assume $\lambda < \risk^*(\xv)$.} Then the excess risk can be estimated in the following way:
    \begin{align*}
        & 2 \EE_{\dataset} \left \{ \max_{c \in C} |\eta_c(\xv) - \hat \eta_c(\xv)| \indicator\{\alpha(\xv) = 0\} \right \} \leqslant 2 \left [R \wedge \PP \left ( \alpha(\xv) =  0\right ) \right ], \\
        & \Delta \cdot \PP \left ( \alpha (\xv) \neq \alpha^*(\xv) \right ) = \Delta \PP \left (\alpha(\xv) =  0 \right ), \\
        & \PP( \alpha(\xv) = 0) = \PP \left (\sum_i K_h(X_i - \xv) > 0 \text{ and } \min_c(1 - \hat \eta_c(\xv)) \leqslant \lambda - z_{1 - \beta/|\classes|} \Vert K \Vert_2 \sqrt{\frac{\hat \eta(\xv) \bigl(1 - \hat \eta(\xv)\bigr)}{N h^d \hat p(\xv)}} \right ).
    \end{align*}

    The event from the RHS implies that there is $c \in C$ such that 
    \begin{align}
        \hat \eta_c(\xv) - \eta_c(\xv) \geqslant \Delta,
    \end{align}
    and, consequently,
    \begin{align*}
        \PP(\alpha(\xv) = 0) \leqslant \sum_{c \in \classes} \PP \left (
            \hat \eta_c(\xv) - \eta_c(\xv) \geqslant \Delta \text{ and } \sum_i K_h(X_i - \xv) > 0
        \right ).
    \end{align*}
    The upper bound was obtained using Lemma~\ref{lemma: Bernstein bounding}.
    
    \textbf{Finally, consider the case $\risk^*(\xv) \geqslant \lambda$.} Then, we estimate
    \begin{align*}
        \EE_{\dataset} \{\max_c |\hat \eta_c(\xv) - \eta_c(\xv)| \indicator\{\alpha(\xv) = 0\}\} \leqslant R |\classes|,
    \end{align*}
    and
    \begin{align*}
        \PP (\alpha(\xv) = 1) & \leqslant  \sum_{c \in \classes} \PP \left (
            \eta_c(\xv) - \hat \eta_c(\xv) \geqslant \Delta - z_{1 - \beta/|\classes|} \Vert K \Vert_2 \max_c \sqrt{
                \frac{\hat \eta_c(\xv) (1 - \hat \eta_c(\xv))}{N h^d \hat p(\xv)}
            }
            \text{ or } \hat p(\xv) = 0
        \right ) \\
        & \leqslant \sum_{c \in \classes}
        \PP \left (
            \eta_c(\xv) - \hat \eta_c(\xv) \geqslant \Delta - z_{1 - \beta/|\classes|}\Vert K \Vert_2
            \sqrt{
                \frac{1}{2 N h^d p(\xv)}
            }
            \text{ and } \hat p(\xv) > 0
        \right ) \\
        & +
        |\classes| \PP \left ( \hat p(\xv) \leqslant \frac{p(\xv)}{2} \right ).
    \end{align*}
    Similarly to Lemma~\ref{lemma: Bernstein bounding}, we bound the last probability by Bernstein's inequality:
    \begin{align*}
        \PP \left ( \hat p(\xv) \leqslant \frac{p(\xv)}{2} \right )
        \leqslant
        \indicator \left \{ \frac{1}{2} \leqslant h^2 \kappa_p \right \}
        \vee
        \exp \left (
        - \frac{\frac{1}{2} h^d N p(\xv) \left (\frac{1}{2} - h^ \kappa_p \right ) }{
            \Vert K \Vert_2^2 + \frac{h L_d}{p(\xv)} \int_{\RR^d} \Vert \vt \Vert_2 K^2(\vt)
            + \max_\vt K(\vt) (1/2 - h^2 \kappa_p)
        }
        \right ).
    \end{align*}
  \end{proof}

  \begin{lemma}
  \label{lemma: Bernstein bounding}
    Suppose all conditions of Theorem~\ref{theorem: main theorem} holds. Then, for any non-negative $r$
    \begin{align*}
        & \PP (\hat \eta_c(\xv) - \eta_c(\xv) \geqslant r \text{ and } \sum_i K_h(X_i - \xv) > 0)
        \leqslant \\
        & \; \leqslant \indicator\{r \leqslant h^2 \kappa\} \vee 
        \exp \left \{- \frac{1}{2} \frac{Nh^d p(\xv) (r - h^2 \kappa)^2}{(1 + r)^2 \left ( \Vert K \Vert_2^2 + \frac{h L_d}{p(\xv)} \int_{\RR^d} \Vert \vt \Vert_2 K^2(\vt) d \vt \right ) + (\max_\vt K(\vt) + r) |r - h^2 \kappa|} \right \}
    \end{align*}
    and
    \begin{align*}
        & \PP (\eta_c(\xv) - \hat \eta_c(\xv) \geqslant r \text{ and } \sum_i K_h(X_i - \xv) > 0)
        \leqslant \\
        & \; \leqslant \indicator\{r \leqslant h^2 \kappa\} \vee 
        \exp \left \{- \frac{1}{2} \frac{Nh^d p(\xv) (r - h^2 \kappa)^2}{(1 + r)^2 \left ( \Vert K \Vert_2^2 + \frac{h L_d}{p(\xv)} \int_{\RR^d} \Vert \vt \Vert_2 K^2(\vt) d \vt \right ) + (\max_\vt K(\vt) + r) (r - h^2 \kappa)} \right \}.
    \end{align*}
  \end{lemma}
  \begin{proof}
    Let us prove the first inequality, the second one can be proved in the same way. Since $\hat p(\xv) > 0$, we can multiply both sides of the inequality $\hat \eta_c(\xv) - \eta_c(\xv) \geqslant r$ by $\sum_{i} K_h(X_i - \xv)$. Thus, the inner event implies
    \begin{align*}
        \sum_i \indicator\{Y_i = c\} K_h(X_i - \xv) - \eta_c(\xv) K_h(X_i - \xv) - r K_h(X_i - \xv) \geqslant 0.
    \end{align*}
    Define
    \begin{align*}
        e_i & = \indicator\{Y_i = c \} K_h(X_i - \xv) - \eta_c(\xv) K_h(X_i - \xv) - r K_h(X_i - \xv), \\
        e & = \EE e_i = \EE \{ \eta_c(X_i) - \eta_c(\xv) \} K_h(X_i - \xv) - r \cdot \EE K_h(X_i - \xv).
    \end{align*}
    In order to write a concentration we should analyze $e$. Inequalities
    \begin{align*}
        \Vert \nabla \eta_c(\xv) - \nabla \eta_c(\vy) \Vert \leqslant \sqrt{d} M_\eta \Vert \xv - \vy \Vert, \\
        \left |
            \int_{0}^\lambda \nabla_i \eta_c(\xv + s e_i) ds
        \right |
        \leqslant 
        \left |
            \eta_c(\xv + \lambda e_i) - \eta_c(\xv)
        \right |
        \leqslant 1,
    \end{align*}
    which hold for each $\lambda$, $\xv$ and $\vy$, guarantee us that the norm of the gradient $\nabla \eta_c(\xv)$ is bounded by
    \begin{align*}
        L_\eta = 2 d \sqrt{M_\eta}.
    \end{align*}
    Moreover, non-negativity of $p(\xv)$, the $L_d$-Lipschitz property and its $L_1$-norm imply that $p(\xv)$ is bounded by $2 L_d d$. Then, Taylor's expansion delivers the following:
    \begin{align*}
        \left | 
            \EE \{\eta_c(X_1) - \eta_c(\xv)\} K_h(X_1 - \xv)
        \right | & = \left | \int_{\RR^d} (\eta(\xv') - \eta(\xv)) K_h(\xv' - \xv) p(\xv') d\xv' \right | \\
        & \leqslant \left |
            h \int_{\RR^d} \langle \nabla \eta (\xv),  t \rangle K(t) p(\xv + h t) dt
        \right | + h^2  d M_\eta L_d \int_{\RR^d} \Vert \vt \Vert_2^2 K(\vt) d\vt \\
        & \leqslant h^2 L_d  d \left (
            \sqrt{M_\eta} + M_\eta
        \right ) \int_{\RR^d} \Vert \vt \Vert_2^2 K(\vt) d \vt.
    \end{align*}
    Similarly,
    \begin{align*}
        \left |
            \EE K_h(X_i - \xv) - p(\xv)
        \right |
        \leqslant 
        \frac{h^2}{2} M_d \int_{\RR^d} \Vert \vt \Vert_2^2 K(\vt) d \vt.
    \end{align*}
    Thus,
    \begin{align*}
        (-e) \geqslant p(\xv) r - \frac{h^2}{2} \left (
            M_d + 2 d M_\eta L_d + 2 d L_d \sqrt{M_\eta}
        \right ) \int_{\RR^d} \Vert \vt \Vert_2^2 K(\vt) d\vt = p(\xv) (r - h^2 \kappa).
    \end{align*}
    If $e > 0$ we estimate the probability by 1. Otherwise, we utilize Bernstein's inequality:
    \begin{align}
        \PP \left ( 
            \sum_i e_i
            -
            N e
            \geqslant N (-e)
        \right )
        & \leqslant
        \exp \left (
            - \frac{
                \frac {1}{2} N^2 e^2
            }{
                \sum_i \operatorname{Var} e_i
                +
                \frac{1}{3} h^{-d} \max_t K(t) N (-e)
            }
        \right ) \nonumber \\
        & \leqslant
        \exp \left (
            - \frac{
                \frac {1}{2} N e^2
            }{
                \EE e_1^2
                +
                \frac{1}{3} h^{-d} \max_t K(t) (-e)
            }
        \right ) \nonumber \\
        & \leqslant 
        \exp \left (
            - \frac{
                \frac {1}{2} N e^2
            }{
                (1 + r)^2 \EE K_h^2(X_i - \xv)
                +
                h^{-d} (\max_\vt K(\vt) + r) (-e)
            }
        \right ). \label{eq: bernstein exponent}
    \end{align}
    We estimate $\EE K_h^2(X_i - \xv)$ as follows:
    \begin{align*}
        \EE K_h^2(X_i - \xv) = h^{-d} \int_{\RR^d} K^2(\vt) p(\xv + h\vt) dt \leqslant h^{-d} \left (p(\xv) \Vert K \Vert_2^2 + h L_d \int_{\RR^d} \Vert \vt \Vert K^2(\vt) d \vt \right ).
    \end{align*}
  \end{proof}

\subsection{Toy Example of Classification with Reject Option}
\label{sec:reject_example}
  For illustration, we present some experimental results with a toy example, see Figure~\ref{fig: nuq vs plug-in experiments}. We consider the smoothed step function as $\eta(\xv)$, while data $X_i$ is distributed according to the normal distribution with the mean $0.5$ and variance~$0.04$ (see Figure~\ref{subfig: nuq vs plug-in data distribution}). We study the point-wise excess risk of NUQ and the plug-in rule. For the points with high covariate mass (Figures~\ref{fig: nuq vs plug-in experiments}c and~~\ref{fig: nuq vs plug-in experiments}d) methods show comparable results. NUQ is useful for points lying with low covariate density, see Figure~\ref{fig: nuq vs plug-in experiments}e. However, for points without any label noise (Figure~\ref{fig: nuq vs plug-in experiments}b) plug-in is still better as it quickly learns the correct class while NUQ is more conservative.
 
 \begin{figure}[!h]
     \centering
     \begin{subfigure}{0.45\textwidth}
        \centering
        \includegraphics[width=\textwidth]{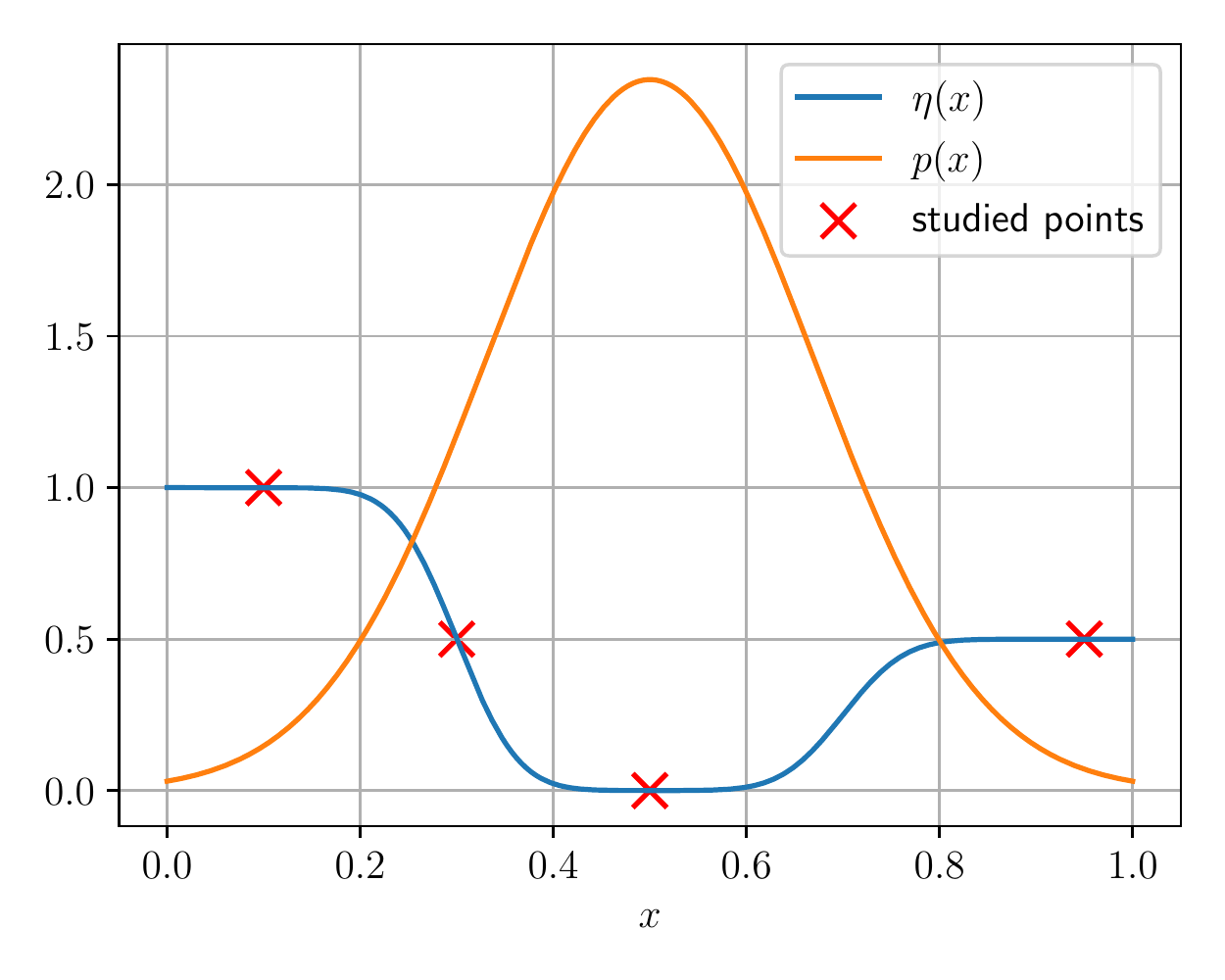}
        \caption{data}
        \label{subfig: nuq vs plug-in data distribution}
     \end{subfigure}
     \begin{subfigure}{0.48\textwidth}
        \centering
        \includegraphics[width=\textwidth]{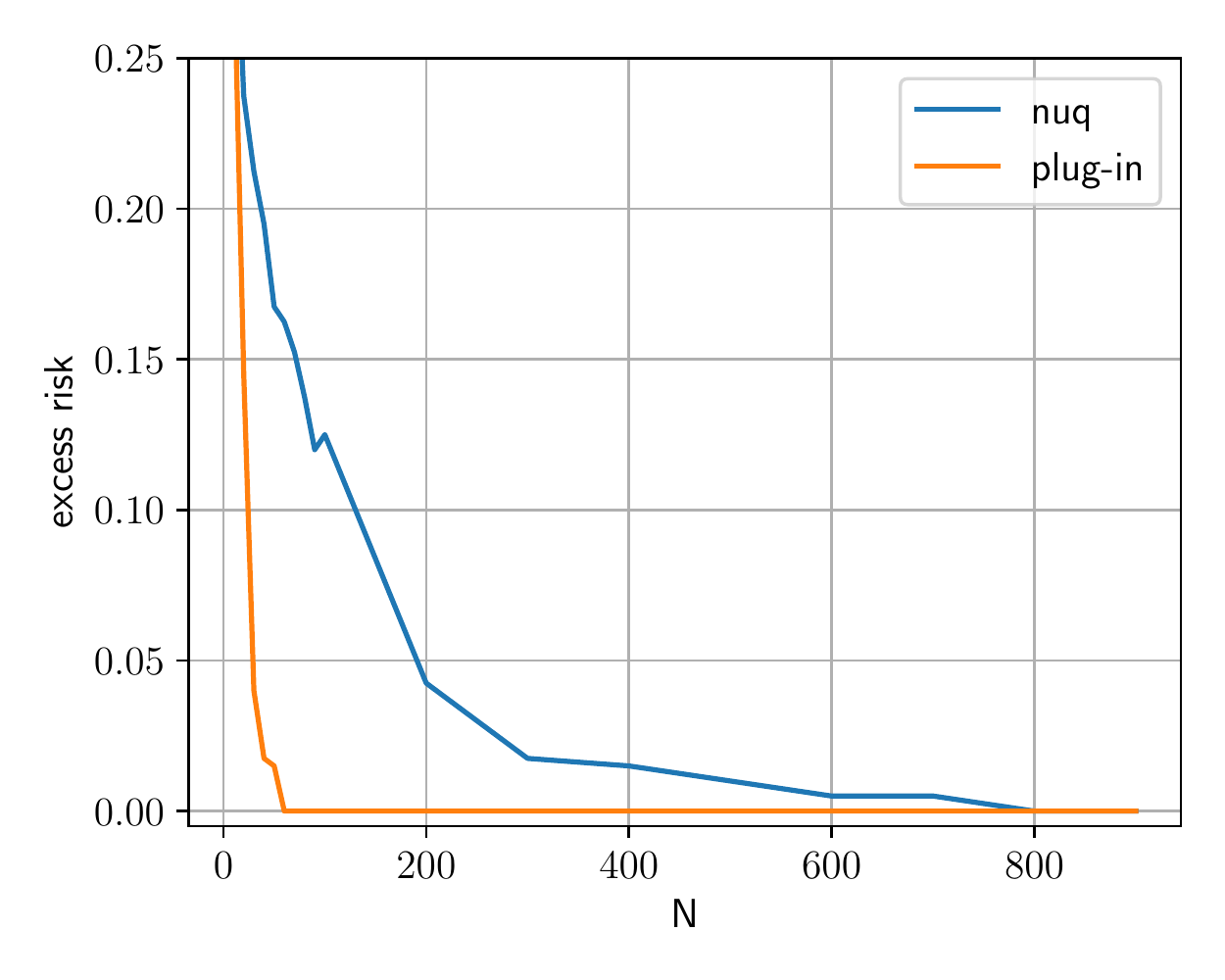}
        \caption{$x=0.1$}
     \end{subfigure}
     \begin{subfigure}{0.48\textwidth}
        \centering
         \includegraphics[width=\textwidth]{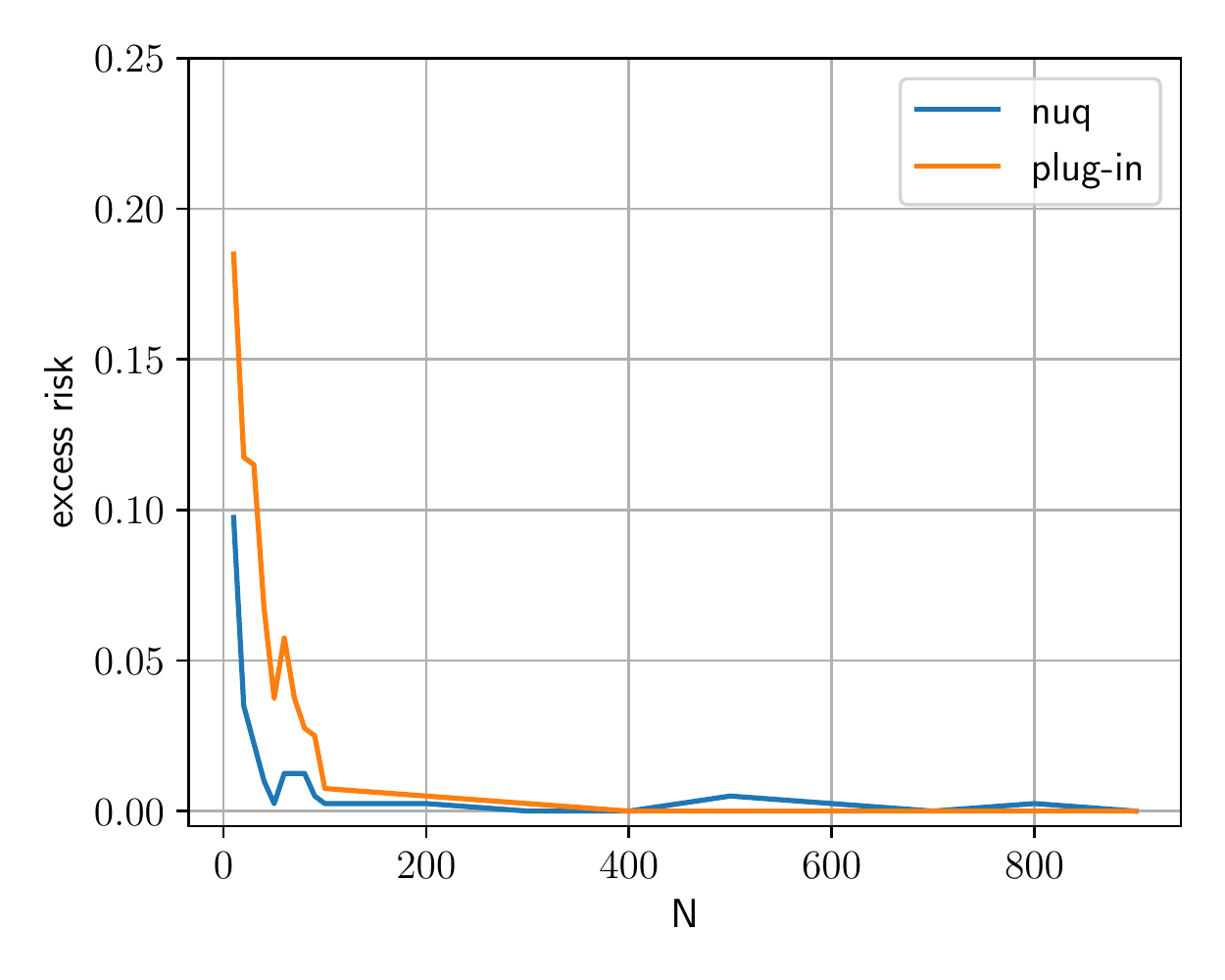}
         \caption{$x=0.3$}
     \end{subfigure}
     \begin{subfigure}{0.48\textwidth}
        \centering
         \includegraphics[width=\textwidth]{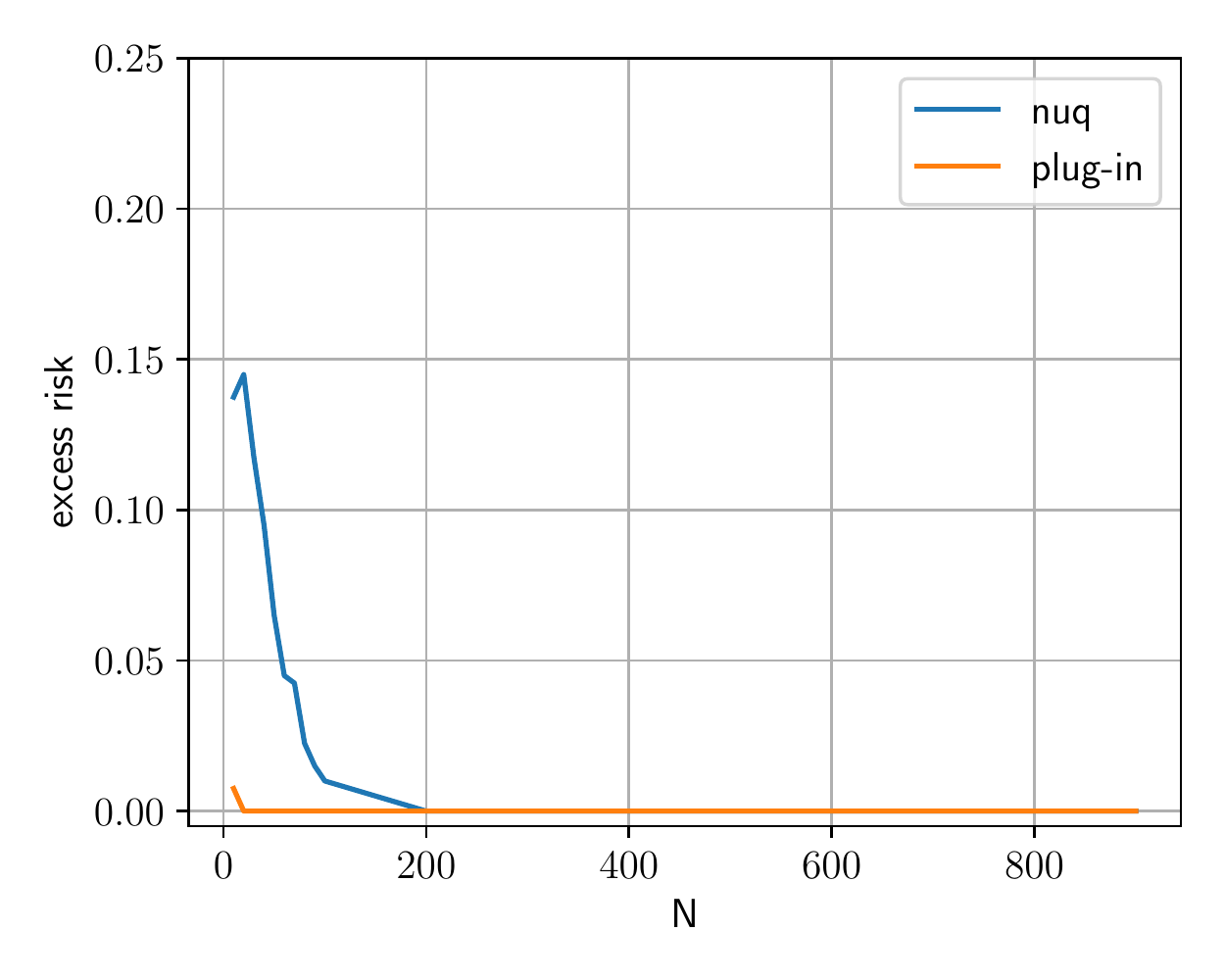}
         \caption{$x=0.5$}
     \end{subfigure}
     \begin{subfigure}{0.48\textwidth}
        \centering
         \includegraphics[width=\textwidth]{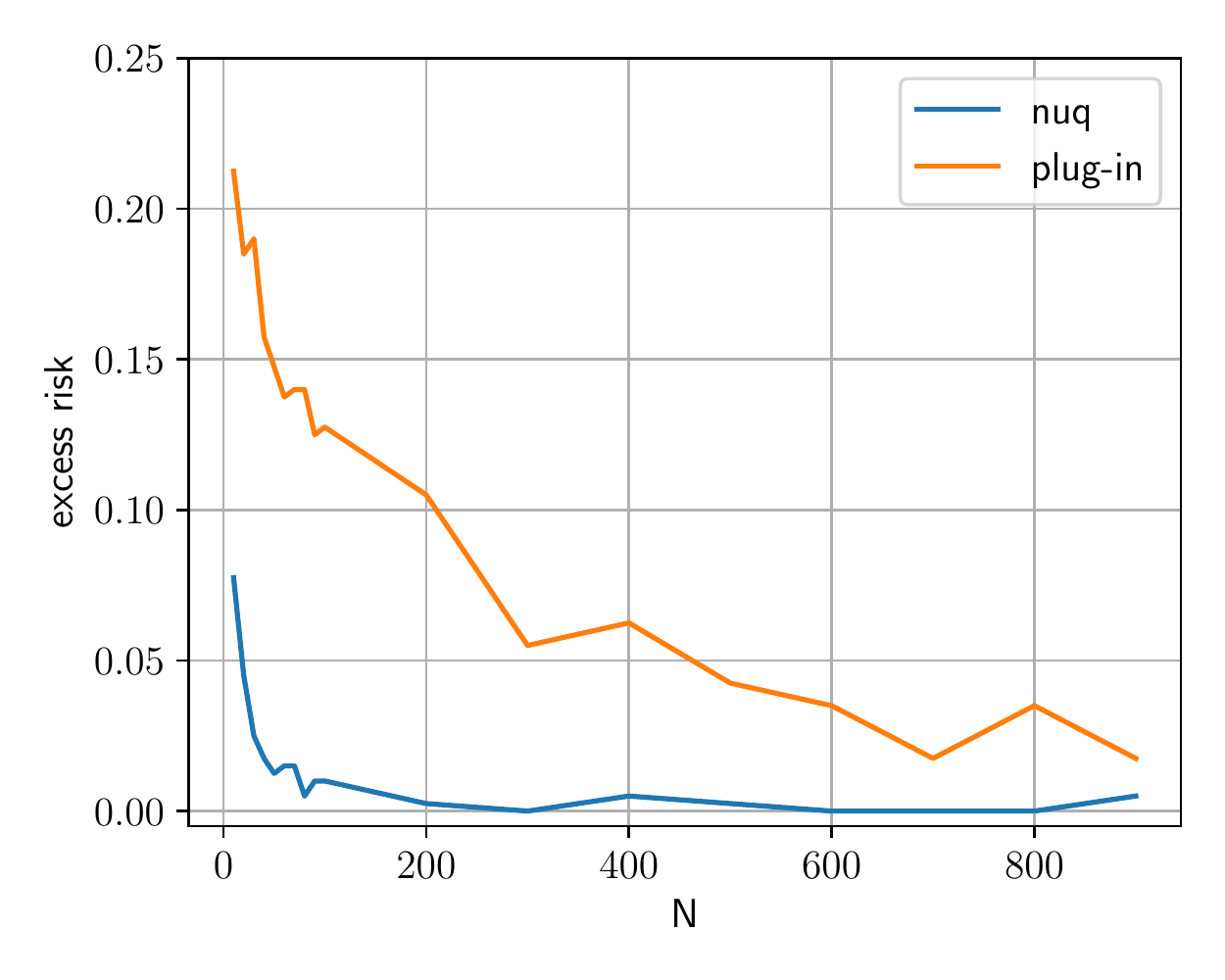}
         \caption{$x=0.95$}
     \end{subfigure}
     \caption{Excess risk at studied points marked on Figure~\ref{subfig: nuq vs plug-in data distribution}. We choose RBF-kernel and quantile $1-\beta$ to be equal~$0.95$.}
     \label{fig: nuq vs plug-in experiments}
 \end{figure}

\section{Toy Experiment on Detecting Actual Aleatoric and Epistemic Uncertainties}
\label{appendix:toy_example_actual_uncertainties}
  In this section, we conduct a toy experiment, for which we explicitly know what should be the true probability of class one, as well as the true data density.

  Let us consider a binary classification problem. Our dataset consists of 5000 samples from three different one-dimensional Gaussians, located to mix classes. Colors denote class label: red -- 0; green -- 1, see Figure~\ref{fig:toy_example_appendix_gaussians}.
  For this particular data model, we can compute the conditional probability of a data point $\xv$ belongs to class 1: $\eta(\xv) = p(Y = 1 \mid X = \xv)$. We build an estimate of this conditional using our Nadaraya-Watson kernel-based approach $\hat{\eta}(\xv)$.
  Further, we generate a uniform grid, and for each point of this grid, using our method, we can upper bound difference between the true conditional and our approximation. This difference, according to our approach, is considered as an epistemic uncertainty (see Figure~\ref{fig:toy_example_appendix_epistemic}). The green line in this plot denotes an absolute difference between the true conditional and our approximation. The red line denotes our epistemic uncertainty. From the picture, we can see that our epistemic uncertainty approximates the probabilities difference well.
  Next, we show how our aleatoric uncertainty relates to the true class 1 conditional probability. In the Figure~\ref{fig:toy_example_appendix_aleatoric} we show true conditional distribution $\eta(\xv)$ (orange line) and our approximation of the aleatoric uncertainty $\min\{\hat{\eta}(\xv), 1 - \hat{\eta}(\xv)\}$ (red line). We can see that our approximation is high exactly in the same regions where the true conditional is absolutely unsure about the class label.

  \begin{figure*}[!ht]
    \begin{subfigure}{0.3\textwidth}
      \includegraphics[width=\linewidth]{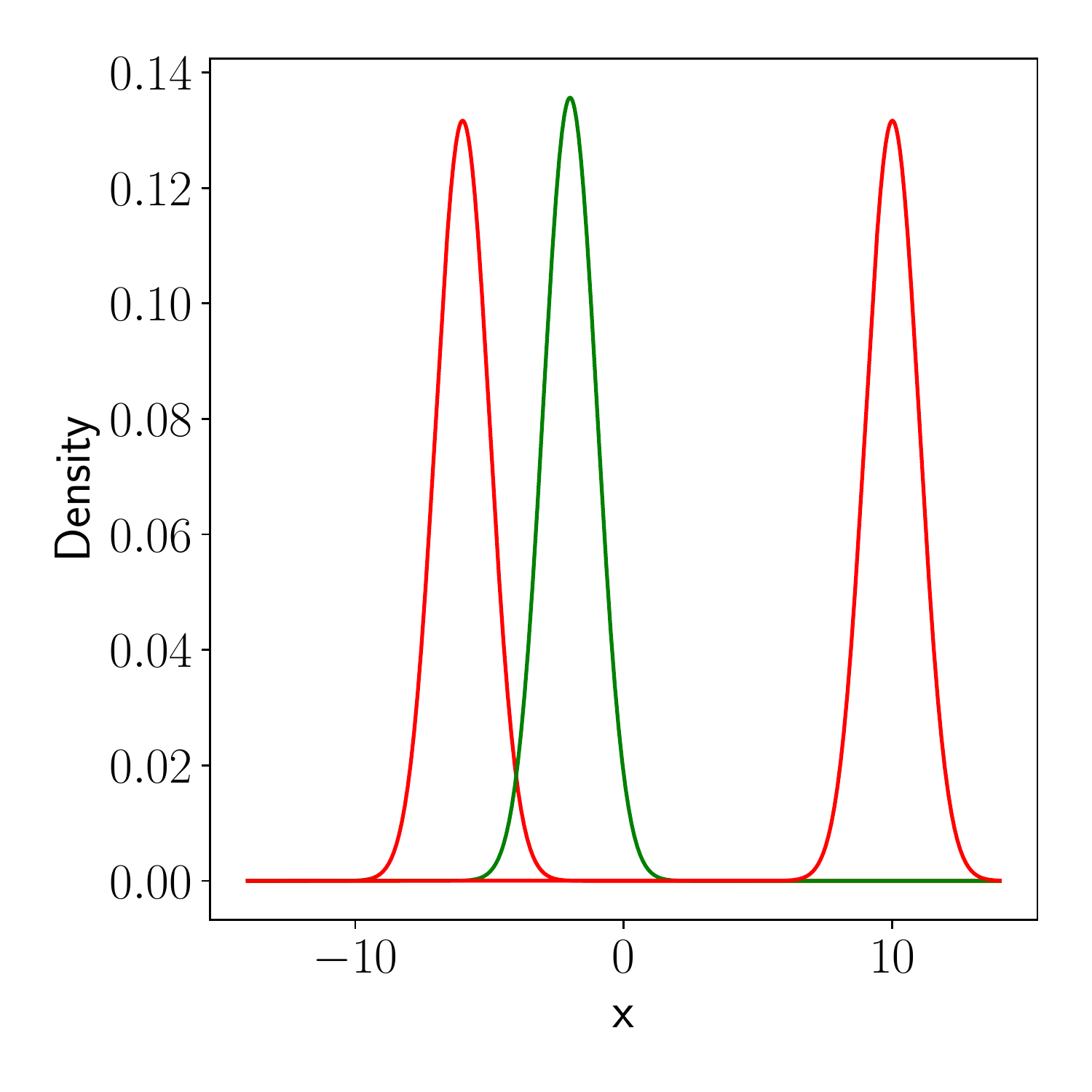}
      \caption{}
      \label{fig:toy_example_appendix_gaussians}
    \end{subfigure}%
    \hspace*{\fill}   
    \begin{subfigure}{0.3\textwidth}
      \includegraphics[width=\linewidth]{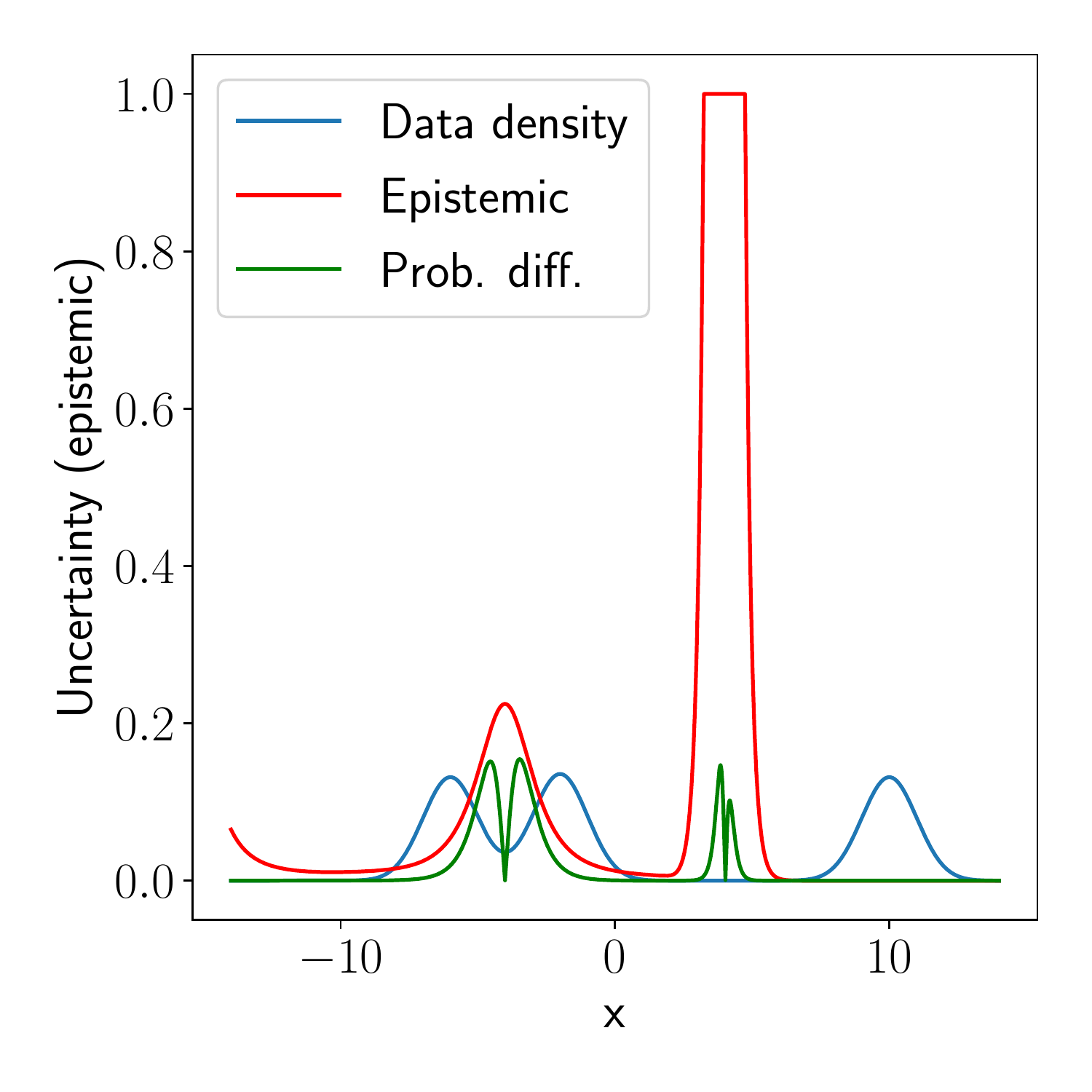}
      \caption{}
      \label{fig:toy_example_appendix_epistemic}
    \end{subfigure}%
    \hspace*{\fill}   
    \begin{subfigure}{0.3\textwidth}
        \includegraphics[width=\linewidth]{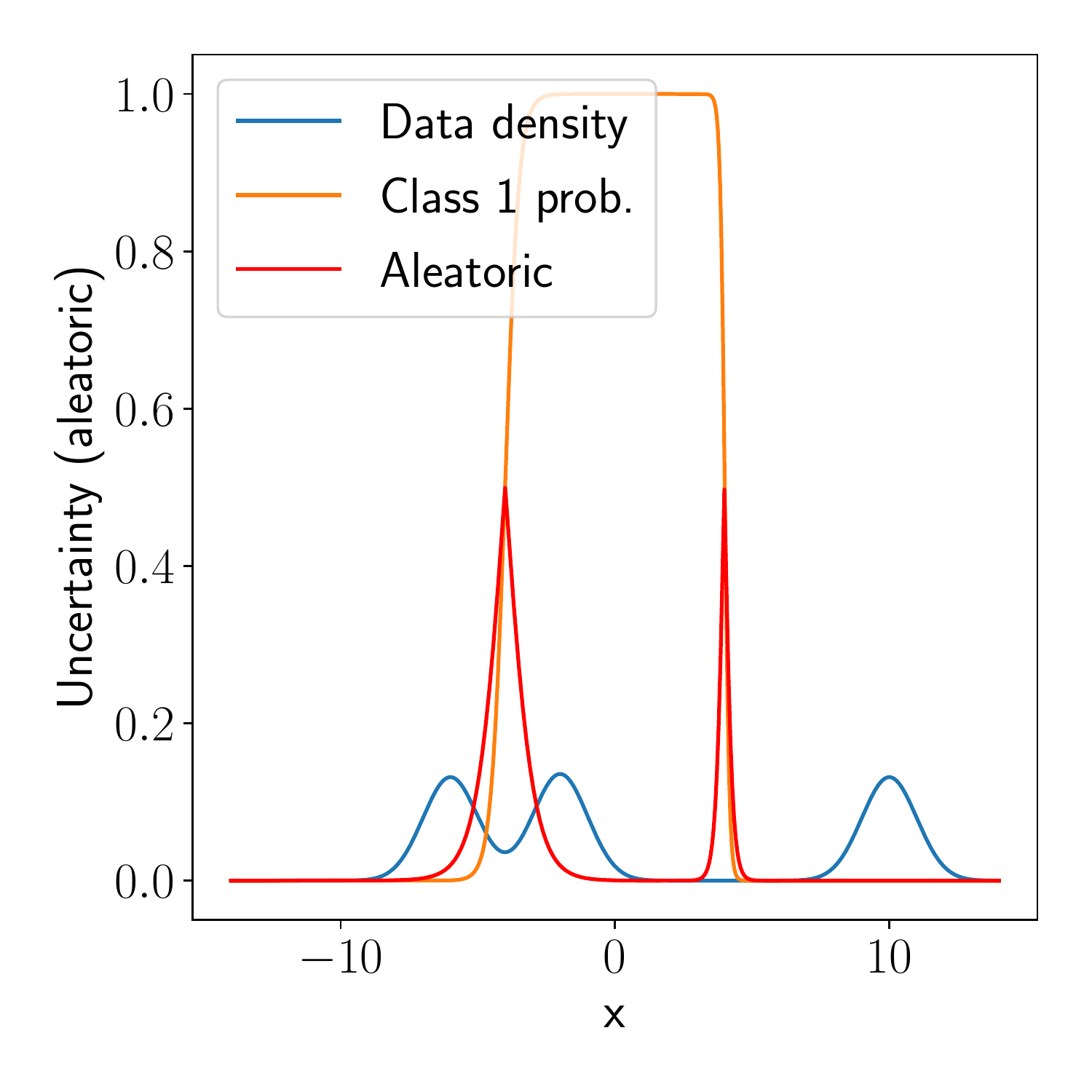} 
        \caption{}
        \label{fig:toy_example_appendix_aleatoric}
    \end{subfigure}

    \caption{(a): Mixture of one dimensional Gaussians we took samples from. Color denotes class label. (b): Epistemic uncertainty our model assigns to data points. Note that the uncertainty is quite high in the region of 3-5. For the sake of visualization, we clipped the maximum value to be 1. (c): Our approximation of aleatoric uncertainty is built along with the true conditional probability.}
  \label{fig:toy_example_appendix}
  \end{figure*}

\section{Additional Experiments on Image Datasets}
\label{sec:additional_image}

\subsection{Rotated MNIST}
\label{appendix:rotated_mnist}

  \begin{figure*}[!t]
    \centering
    \begin{subfigure}{0.33\linewidth}
      \centering
      \includegraphics[width=\textwidth]{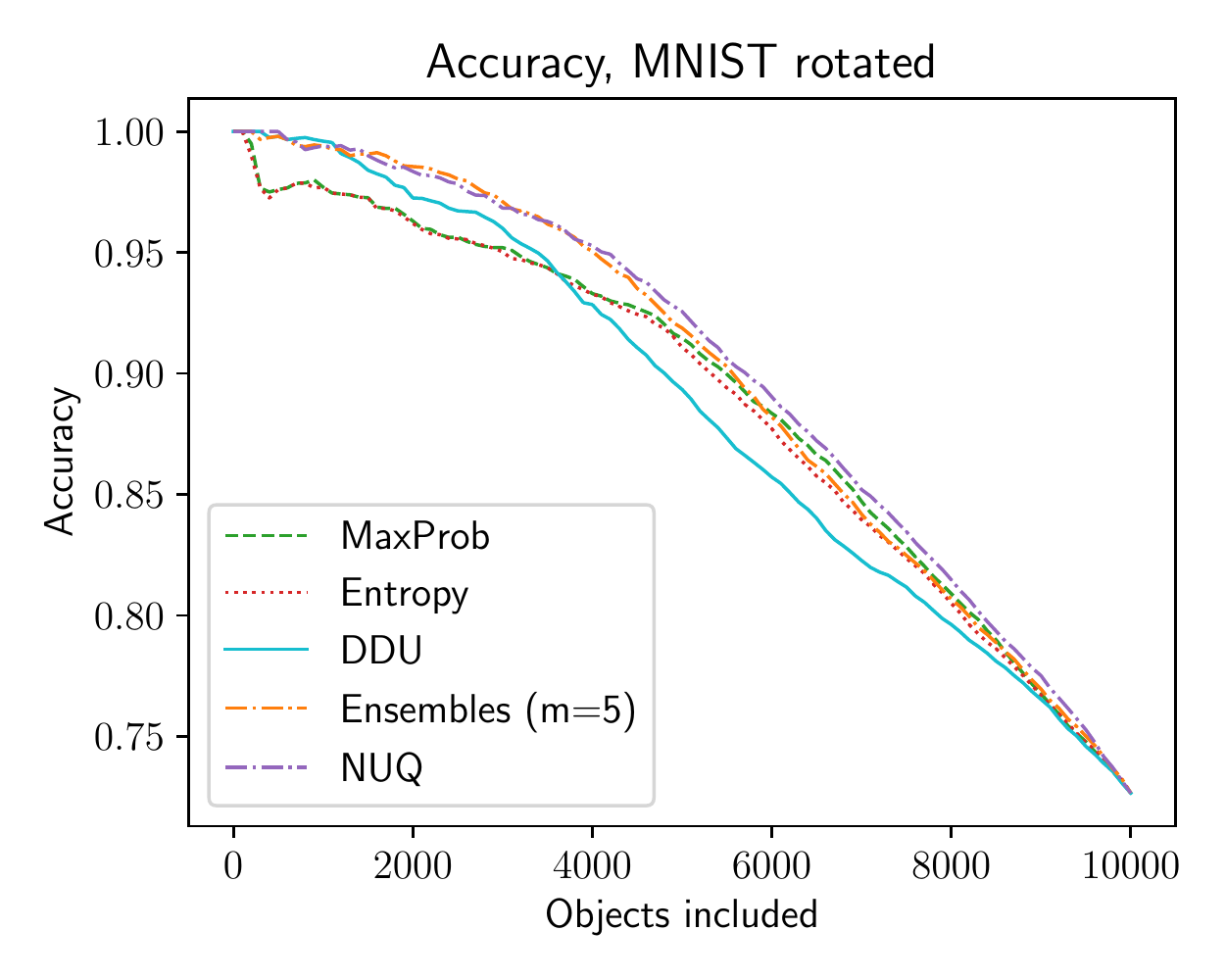}
    \vspace{-20pt}
    \caption{}
    \label{fig:mnist_accuracy_ensembles}
    \end{subfigure}%
    \hspace*{\fill}   
    \begin{subfigure}{0.33\linewidth}
      \centering
      \includegraphics[width=\textwidth]{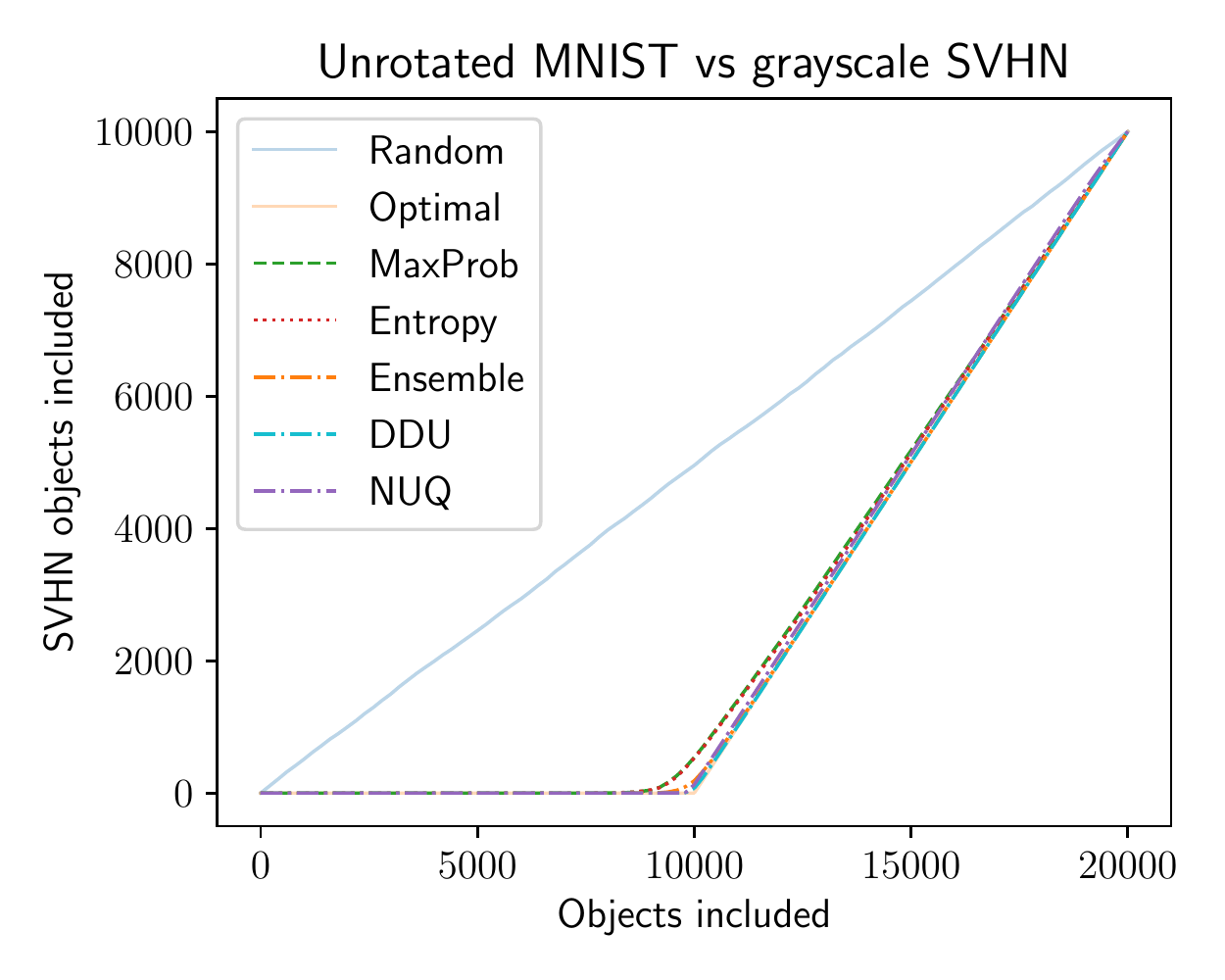}
    \vspace{-20pt}
     \caption{}
     \label{fig:mnist_unrotated_ensembles}
    \end{subfigure}%
    \hspace*{\fill}   
    \begin{subfigure}{0.33\linewidth}
    \centering
        \includegraphics[width=\textwidth]{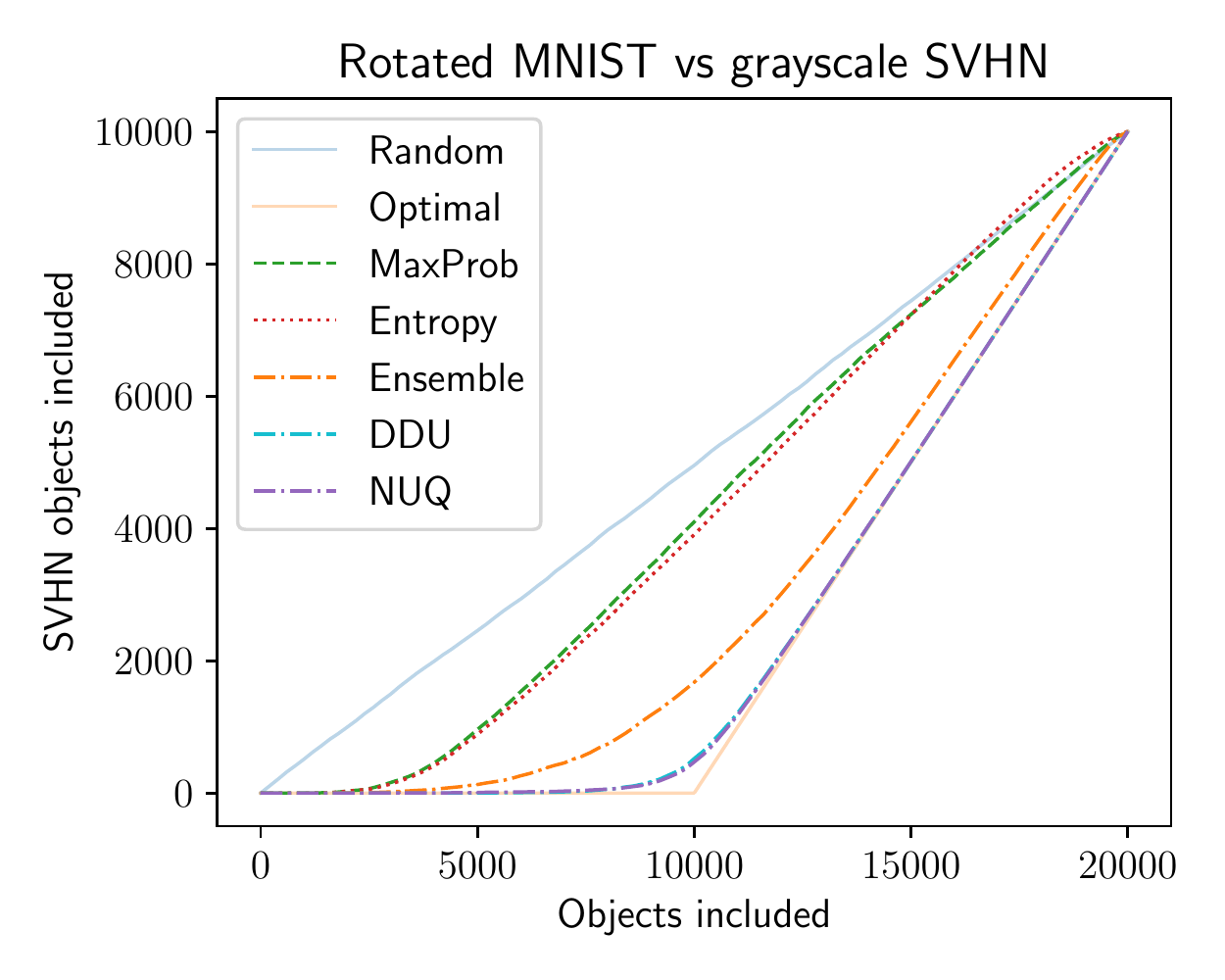}
    \vspace{-20pt}
    \caption{}
    \label{fig:mnist_rotated_ensembles}
    \end{subfigure}
    \vspace{-10pt}
    \caption{(a) Accuracy for images sorted by uncertainty on rotated MNIST. (b) Number of SVHN images included into consideration vs unrotated MNIST. In this simpler version, even the basic entropy manages to achieve a good result. (c) More challenging task -- number of SVHN images included into consideration vs rotated MNIST. NUQ still distinguish between datasets with close to an optimal solution.}
  \label{fig:mnist_ensembles}
  \end{figure*}

  The first example is classification under covariate shift on MNIST~\citep{lecun2010mnist}. We train a small convolutional neural network with three convolution layers, see Section~\ref{sec:architechtures}. This is the base model we use to obtain logits for the input objects. We consider a particular instance of distribution shift for evaluation by using a test set of MNIST images rotated at a random angle in the range from 30 to 45 degrees. This set contains 10000 images. The range of angles reassures that the data does not look like the original MNIST data, though many resulting pictures can still remind the ones from training. 

  This experiment considers two simple baselines: MaxProb and Entropy-based uncertainty estimates of the base model. We compliment them with two more challenging baselines: Deep ensemble and DDU~\citep{mukhoti2021deterministic} (as the most successful among deterministic methods). We compare them all with NUQ-based estimate of epistemic uncertainty \(\hat{\textbf{U}}_e(\xv)\).
  To evaluate the quality of the uncertainty estimates, we sort the objects from the test dataset in order of ascending uncertainties. 
  
  Then we obtain the model's predictions and plot how accuracy changes with the number of objects taken into consideration; see Figure~\ref{fig:mnist_accuracy_ensembles}. The valid uncertainty estimation method is expected to produce the plot with accuracy decreasing when more objects are taken into account. Moreover, the higher is the plot, the better is the quality of the corresponding uncertainty estimate.
  
  We see that the plots for all the considered methods show the expected trend, while uncertainties obtained by NUQ are more reliable. NUQ distinctly outperforms DDU and has comparable performance with deep ensembles.

\subsection{MNIST vs. SVHN}
\label{appendix:mnist_svhn}
  To make the problem more challenging, we consider the SVHN dataset~\citep{svhn}, convert it to grayscale, and resize it to the shape of 28x28. The size of this additional SVHN-based dataset is again 10000. We take the base model trained on MNIST from the previous section and consider the problem of OOD detection with SVHN being the OOD dataset. As in-distribution data, we first consider the test set of 10000 MNIST images. We again compute uncertainties for each object on this concatenated dataset (10000 of MNIST and 10000 of SVHN) and sort them by their uncertainties in ascending order. 
  The goal for uncertainty quantification methods is to sort all objects so that all MNIST images have lower uncertainty values than SVHN ones.
  Note that the optimal decision rule in this case is a ReLU-shaped function, with a break at point 10000.
  
  For NUQ we use epistemic uncertainty \(\hat{\textbf{U}}_e(\xv)\) in this experiment. In Figure~\ref{fig:mnist_unrotated_ensembles} we plot the number of objects included from the SVHN dataset. It is seen that Ensembles, NUQ, and DDU outperform MaxProb and Entropy. All of these three methods perform almost as an optimal decision.

  Next, we consider more challenging problem of separation between rotated MNIST (see Section~\ref{appendix:rotated_mnist}) and SVHN. We expect it is harder to distinguish between them as rotated MNIST images differ from those used to train the network. However, Figure~\ref{fig:mnist_rotated_ensembles} shows that NUQ still does a very good job and allows for almost perfect separation. 
  Interestingly, this example shows that ensembles are worse than DDU and NUQ. The performance of the last two is visually almost identical.

\subsection{Performance Difference on CIFAR-100 and ImageNet}
\label{sec:imagenet_performance}
  \begin{figure*}[t!]
    \begin{subfigure}{0.48\linewidth}
      \includegraphics[width=\textwidth]{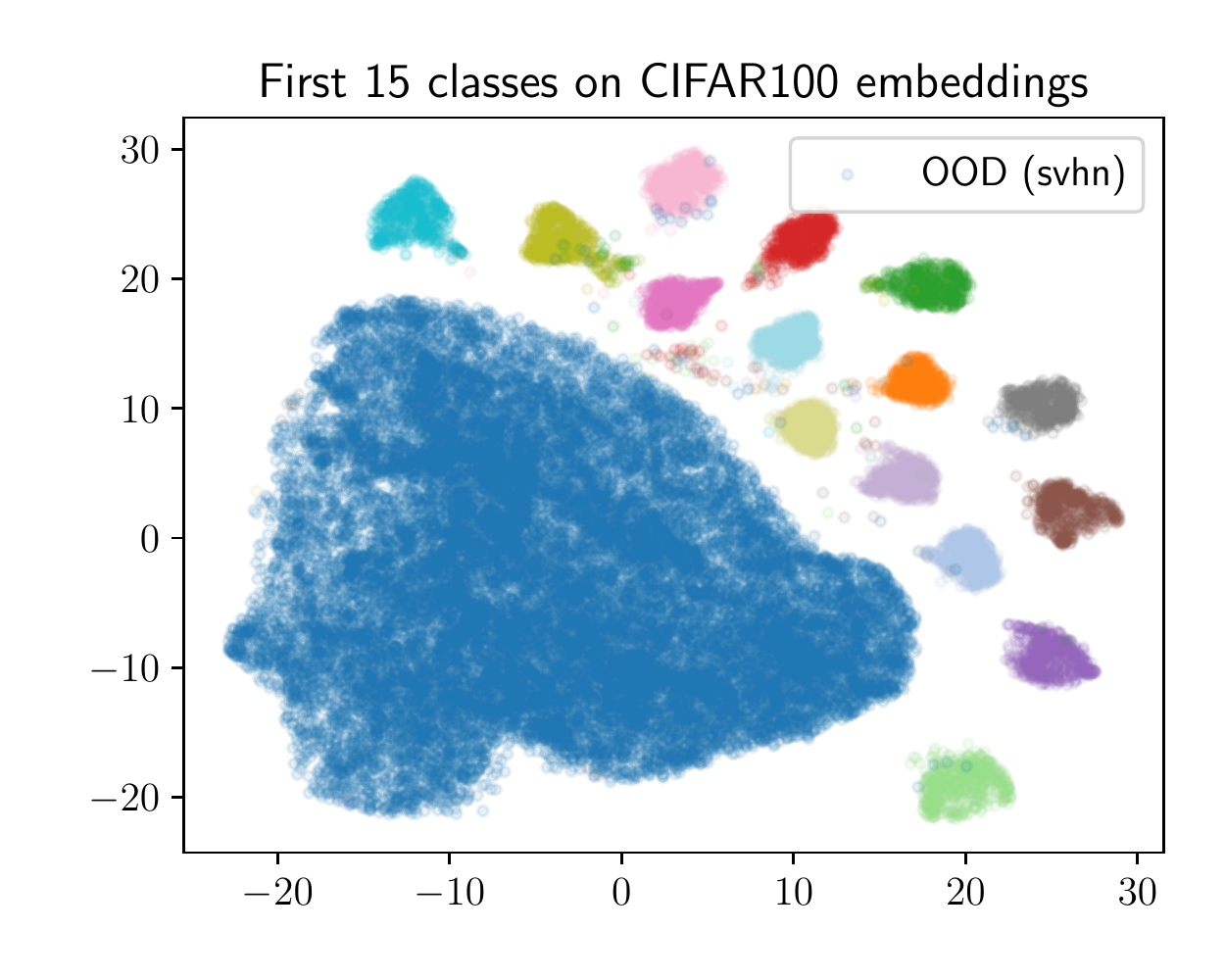}
      \caption{}
    \end{subfigure}%
    \hspace*{\fill}   
    \begin{subfigure}{0.48\linewidth}
      \includegraphics[width=\textwidth]{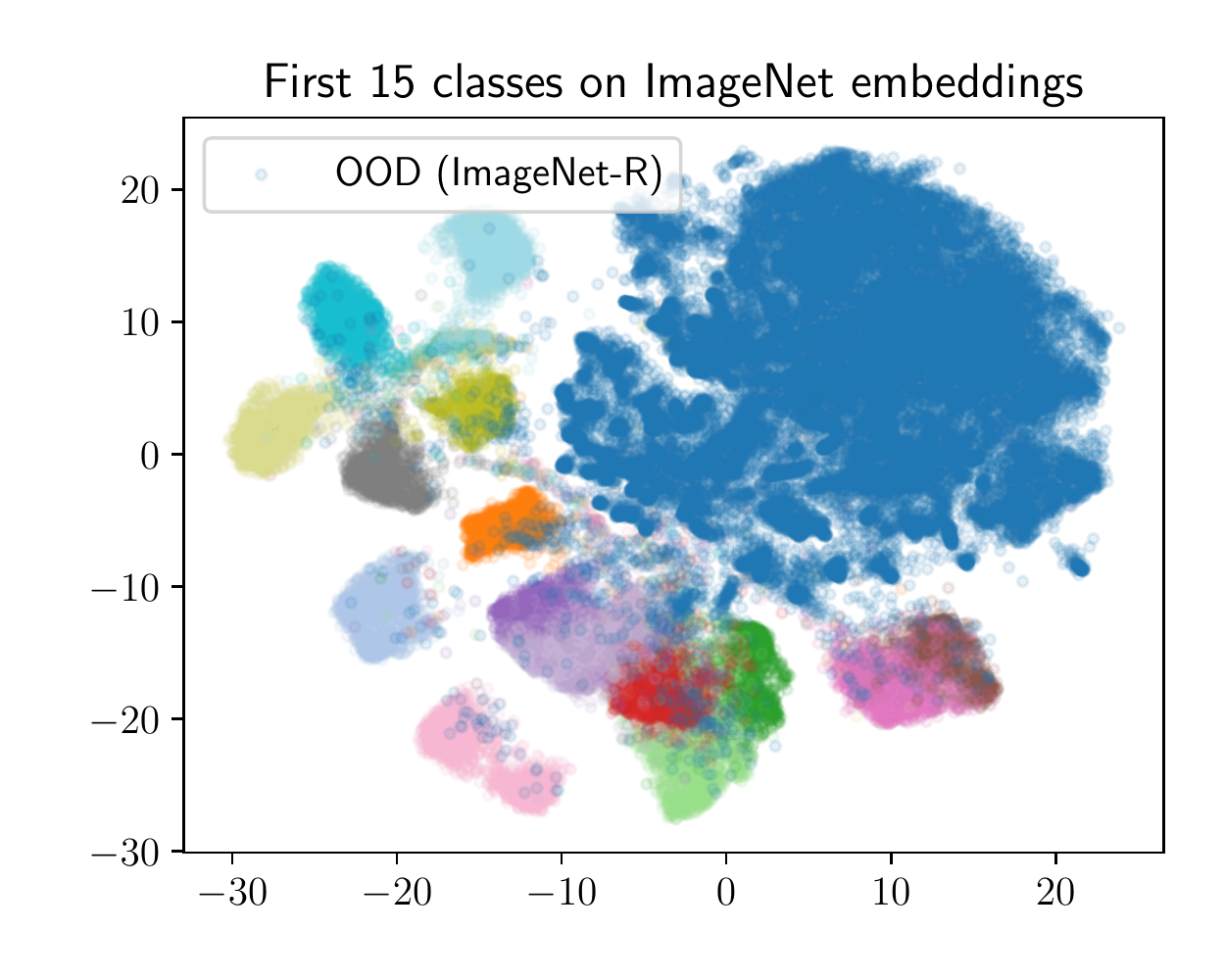}
      \caption{}
    \end{subfigure}%

    \caption{Embeddings space visualization for CIFAR (a) and ImageNet (b). We present the embeddings for first 15 classes on test dataset (in various colors) and all the embeddings for out-of-distribution datasets (in blue). The OOD dataset for CIFAR is SVHN and for ImageNet it is ImageNet-R.}
  \label{fig:ncvis}
  \end{figure*}

  One of the things that caught our attention is the superior performance of NUQ on ImageNet, given that it has very similar results with DDU on CIFAR-100. One of our hypotheses was that embeddings have a more complex and multi-modal distribution for a more complex ImageNet dataset compared to simpler CIFAR-100. To check this, we made t-SNE based embeddings of out-of-distribution and test data (see Figure~\ref{fig:ncvis}). While we understand the limitation of this type of visualization, the ImageNet embeddings appear to be much more irregular compared to well-shaped clusters for CIFAR-100. Because of that, the modeling of the class with a single Gaussian (in DDU) might not work very well for ImageNet. NUQ approach performs the modeling of distributions in a much more flexible way, which is beneficial for approximating complex distributions. We hypothesize that this is the reason for the NUQ's superior performance.

\subsection{Ablation Study on CIFAR-100}
\label{sec:ablation}
  We need an estimator of marginal density \(p(\xv)\) for our method, and there exist different options. We consider kernel method with RBF kernel and logistic kernel and Gaussian mixtures model (GMM). There is also a question about which embeddings to use -- the DDU paper~\cite{mukhoti2021deterministic} proposes to take the features from the second last layer; while logits from the last layer represent a reasonable choice as well. To validate the options, we conducted some ablation study on out-of-distribution detection for the CIFAR-100 dataset, similar to the main experiment.
  
  First, we compare the DDU and NUQ on embeddings from the  second last and last layers (Table~\ref{tab:ablation_layer}) on SVHN, LSUN, and Smooth datasets. Secondly, we compare the NUQ method on RBF, logistic kernel, and GMM for both last and penultimate layer embeddings (Table~\ref{tab:ablation_density}). As we can see from the tables, the optimal option is GMM density on the penultimate layer.

  \begin{table}[ht]
    \centering
    \begin{tabular}{|c|c|c|c|c|}
      \hline
             & DDU, features & DDU, logits & NUQ, features & NUQ,  logits \\ \hline
      SVHN   & 89.6±1.6      & 88.2±0.6    & 89.7±1.6       &   88.2±0.6  \\ \hline
      LSUN   & 92.1±0.6      & 90.9±0.4    & 92.3±0.6      &   90.9±0.4   \\ \hline
      Smooth & 97.1±3.1      & 96.3±4.1    & 96.8±3.8      &    96.2±4.1  \\ \hline
    \end{tabular}
    \caption{Comparison of DDU and NUQ predictions on different type of embeddings: logits (last layer) and features (second last layer).}
  \label{tab:ablation_layer}
  \end{table}
  
  \begin{table}[ht]
    \centering
    \begin{tabular}{|c|c|c|c|c|c|c|}
      \hline
             & RBF, f & RBF, l & Logistic, f & Logistic, l & GMM, f & GMM, l \\
             \hline
      SVHN   & 84.4±3.2      & 84.7±3.1    & 84.8±2.9          & 86.7±2.6         & 89.7±1.6      & 88.2±0.6    \\
      \hline
      LSUN   & 88.2±1.0      & 88.1±0.8    & 88.5±4.0          & 90.3±1.0         & 92.3±0.6      & 90.9±0.4    \\
      \hline
      Smooth & 85.5±6.8      & 87.7±9.4    & 86.2±8.2          & 90.8±7.8         & 96.8±3.8      & 96.2±4.1    \\
      \hline
    \end{tabular}
    \caption{Probability density methods comparison -- radial basis function kernel (RBF), logistic kernel, Gaussian mixture models (GMM). 'f' (features) marks models, built on embeddings from a second last layer and 'l' (logits) is for the ones built on embeddings from a last layer.}
  \label{tab:ablation_density}
  \end{table}

  Kernel-based methods rely on the ``reasonable'' geometry of the embedding space, meaning that embeddings of similar images should be not too far and different images should not collapse into a single point. Our motivation to use spectral normalization during training is to make the embedding space smooth with respect to input images. 
  We have conducted an extra ablation study, comparing the result for feature extractors with and without spectral normalization, see Table~\ref{tab:spectral}. 
  The results confirm our hypothesis, as the spectral-normalized version performs better, though the NUQ beats the baseline even without applying the modification to the ResNet training. We also show here that entropy performs better than maximum probability as an uncertainty measure.

  \begin{table}[ht!]
    \centering
    \begin{tabular}{|c|c|c|c|c|}
      \hline
      OOD dataset  & DDU & DDU (spectral) & NUQ     & NUQ (spectral)     \\ \hline
      SVHN        & 88.7±4.3 & 89.6±1.6 & 86.8±1.2 & 89.7±1.6 \\ \hline
      LSUN        & 91.3±0.9 & 92.1±0.6 & 91.2±1.1 & 92.3±0.6 \\ \hline
      Smooth     & 95.7±1.2 & 97.1±3.1 & 95.5±1.3 & 96.8±3.8 \\ \hline
    \end{tabular}
    \caption{Comparing the influence of spectral normalization on the model performance for OOD detection, ROC-AUC.}
  \label{tab:spectral}
  \end{table}

\subsection{Sensitivity to Ensemble Size}
\label{sec:ensemble:size}
  In this section, we explore the ensemble model performance regarding the number of models. From Table~\ref{tab:ensemble_cifar}, we can see that 5 models is a reasonable amount number for CIFAR-100. For the ImageNet dataset (Table~\ref{tab:ensemble_imagenet}), increasing the number of models gives a steady gain, but still 5 models provide the gain within the error margin.
 
  \begin{table}[ht!]
    \centering
    \begin{tabular}{|l|l|l|l|l|l|}
      \hline
      OOD dataset & \multicolumn{1}{r|}{2} & \multicolumn{1}{r|}{3} & \multicolumn{1}{r|}{5} & \multicolumn{1}{r|}{7} & \multicolumn{1}{r|}{10} \\
      \hline
      SVHN        & 82.3 ± 1.3               & 82.4 ± 0.7               & 82.9 ± 0.9               & 82.7 ± 0.7               & 82.6 ± 0.5                \\
      \hline
      LSUN        & 85.1 ± 0.6               & 85.9 ± 0.6               & 86.5 ± 0.8               & 87.1 ± 0.6               & 87.1 ± 0.6                \\
      \hline
      Smooth      & 83.7 ± 6.5               & 83.4 ± 3.2               & 83.7 ± 1.2               & 83.3 ± 1.5               & 83.2 ± 1.6                \\
      \hline
    \end{tabular}
    \caption{Ablation study for ensemble size on CIFAR100 in-distribution dataset for OOD detection. The number of models above 5 give almost no gain (even some loss sometimes) within an error margin.}
  \label{tab:ensemble_cifar}
  \end{table}

  \begin{table}[ht!]
    \centering
    \begin{tabular}{|l|r|r|r|r|r|}
      \hline
      OOD dataset & 2    & 3    & 5    & 7    & 9    \\
      \hline
      ImageNet-O  & 49.8 & 50.8 & 51.9 & 52.1 & 52.6 \\
      \hline
      ImageNet-R  & 84.7 & 85.3 & 85.8 & 86   & 86.1 \\
      \hline
    \end{tabular}
    \caption{Ablation study for ensemble size on ImageNet out-of-distribution detection task.
    }
  \label{tab:ensemble_imagenet}
  \end{table}

\subsection{Additional Experiments with DUQ}
\label{sec:duq_additional_experiments}
  DUQ~\cite{van2020uncertainty} is one of the baselines in the paper. We chose it as similarly to our method it requires only a single forward pass and uses post-processing for embeddings. In the original article, the method shows a good result on a relatively small dataset CIFAR-10 with a ResNet18 model. We tried to train it on CIFAR-100 and ImageNet with a larger model, ResNet50, but there were difficulties with training process as method failed to converge to the model of reasonable quality. We believe the cause of the problem was gradient penalty as a regularization method, and we switched to spectral normalization instead. DUQ training requires a balance between hyperparameters such as length scale, momentum, and learning rate, so we initially trained the model with a pre-trained feature extractor. With careful selection of parameters, we managed to train end-to-end as well, and we observed improvement in all experiments (see Tables~\ref{tab:duq_cifar} and~\ref{tab:duq}), although methods like DDU and NUQ had more stable training in our experience and a better final result.

  \begin{table}[ht!]
    \centering
    \resizebox{\textwidth}{!}{
    \begin{tabular}{|l|l|l|l|l|l|l|}
      \hline
      OOD dataset & Ensembles & TTA       & DDU               & NUQ               & DUQ Head & DUQ end-to-end \\
      \hline
      SVHN        & 82.9 ± 0.9  & 81.6 ± 1.2  & 89.6 ± 1.6    & \textbf{89.7 ± 1.6} & 83.6 ± 4.0  & 88.7 ± 6.3 \\
      \hline
      LSUN        & 86.5 ± 0.8  & 85.0 ± 2.7  & 92.1 ± 0.6    & \textbf{92.3 ± 0.6} & 87.2 ± 2.1 & 90.8 ± 6.7 \\
      \hline
      Smooth      & 83.7 ± 1.2  & 73.2 ± 10.8 & \textbf{97.1 ± 3.1} & 96.8 ± 3.8    & 83.8 ± 11.4 & 91.1 ± 8.4 \\
      \hline
    \end{tabular}
    }
      \caption{Performance on OOD detection for CIFAR-100}
      \label{tab:duq_cifar}
  \end{table}
  
  \begin{table}[ht!]
  \centering
    \begin{tabular}{|l|r|r|r|r|r|r|r|}
      \hline
      OOD dataset & \multicolumn{1}{l|}{Ensemble} & \multicolumn{1}{l|}{DDU*} & \multicolumn{1}{l|}{NUQ (spectral)*} & \multicolumn{1}{l|}{DUQ Head}  & \multicolumn{1}{l|}{DUQ end-to-end}\\
      \hline
      ImageNet-R                                     & 84.4                          & 74.2                      & \textbf{99.5}                                 & 57.4           & 73.3              \\
      \hline
      ImageNet-O  &  51.9                          & 74.1                      & \textbf{82.4}                                 & 67.3             &71.4            \\
      \hline
    \end{tabular}
    \caption{Performance on OOD detection on ImageNet}
  \label{tab:duq}
  \end{table}

\subsection{Computational Costs}
\label{appendix:computation_costs}
  We benchmarked the training and inference time overhead for our ImageNet experiments. Base model training time corresponds to training of a ResNet-50 feature extractor on a corresponding dataset. The NUQ training time is measured on a full training dataset and includes the parameter search time. The inference time is measured on a full test dataset. Results are presented in Table~\ref{tab:nuq_speed}, where we can see that training overhead is less than 1\% and inference overhead is less than 10\%, while for ensembles, the overhead would be n-fold.

  \begin{table}[h]
    \centering
    \begin{tabular}{|l|c|c|c|}
      \hline
      & Base model & NUQ & Overhead \\
      \hline
      CIFAR-100 training time & 11.5 hours &  42.6 seconds  & 0.10\% \\
      \hline
      CIFAR-100 inference time &  81 seconds  & 1.8 seconds & 2.25\% \\
      \hline
      ImageNet training time & 49 hours & 28 minutes &  0.96\% \\
      \hline
      ImageNet inference time & 225 seconds & 21 seconds &  9.30\% \\
      \hline
    \end{tabular}
    \vspace{0.2cm}
    \caption{Execution time overhead for NUQ on CIFAR-100 and ImageNet datasets}
  \label{tab:nuq_speed}
 \end{table}

\subsection{Choice of Uncertainty Measure for Ensembles}
\label{appendix:choice_of_uncertainty_measure_for_ensemble}
  In the case of ensembles, we have multiple predictions for a single point. It also applies to other methods like test-time augmentation and Monte-Carlo dropout. To get a single value of uncertainty, we must choose the particular uncertainty measure based on . We have tried a few options. Namely, we used mean maximum probability, a standard deviation of predicted probabilities between different runs, mutual information, and entropy. However, in the manuscript, we decided to focus on only one approach (entropy) to reduce clutter, while in our experiments, it performed better on average. We present the results for test time augmentation and ensembles on ImageNet here, see Table~\ref{tab:reduction_study}.

  \begin{table}[]
        \begin{tabular}{|l|lrrr|lrrr|}
        \hline
                   & \multicolumn{4}{l|}{TTA}                                                                                                   & \multicolumn{4}{l|}{Ensemble}                                                                                              \\ \hline
                   & \multicolumn{1}{l|}{MaxProb} & \multicolumn{1}{l|}{Entropy} & \multicolumn{1}{l|}{Std}  & \multicolumn{1}{l|}{MI} & \multicolumn{1}{l|}{MaxProb} & \multicolumn{1}{l|}{Entropy} & \multicolumn{1}{l|}{Std}  & \multicolumn{1}{l|}{MI} \\ \hline
        Imagenet-O & \multicolumn{1}{r|}{29.2}    & \multicolumn{1}{r|}{30.5}    & \multicolumn{1}{r|}{32.3} & 34.8                             & \multicolumn{1}{r|}{48.3}    & \multicolumn{1}{r|}{51.9}    & \multicolumn{1}{r|}{58.5} & 59.1                             \\ \hline
        Imagenet-R & \multicolumn{1}{r|}{82.8}    & \multicolumn{1}{r|}{85.8}    & \multicolumn{1}{r|}{74.1} & 84.7                             & \multicolumn{1}{r|}{81}      & \multicolumn{1}{r|}{84.4}    & \multicolumn{1}{r|}{70.2} & 76.1                             \\ \hline
        \end{tabular}
        \vspace{0.2cm}
        \caption{Study on ensemble and test time augmentation reduction methods on OOD detection. For in-distribution we used test ImageNet dataset and for out-of-distribution ImageNet-O and ImageNet-R were used, respectively. The metric is ROC-AUC.}
       \label{tab:reduction_study}
   \end{table}

\subsection{Detecting OOD Images for ImageNet}
\label{appendix:additional_imagenet_experiments}
  By analogy with Section~\ref{appendix:rotated_mnist}, we conducted experiments with vanilla ImageNet images vs ImageNet-R and ImageNet-O. The idea is to mix images from two datasets and then sort them by uncertainty. For normal images uncertainty tends to be lower, so we can illustratively see the methods performance (Figure~\ref{fig:mixed_imagenet_r}). In case of ImageNet-O, basic models like max probability and ensemble have difficulties, while NUQ and DDU performed reasonably well. For ImageNet-R, NUQ has performance close to optimal.

  \begin{figure}
     \centering
     \begin{subfigure}{0.48\textwidth}
        \centering
        \includegraphics[width=\textwidth]{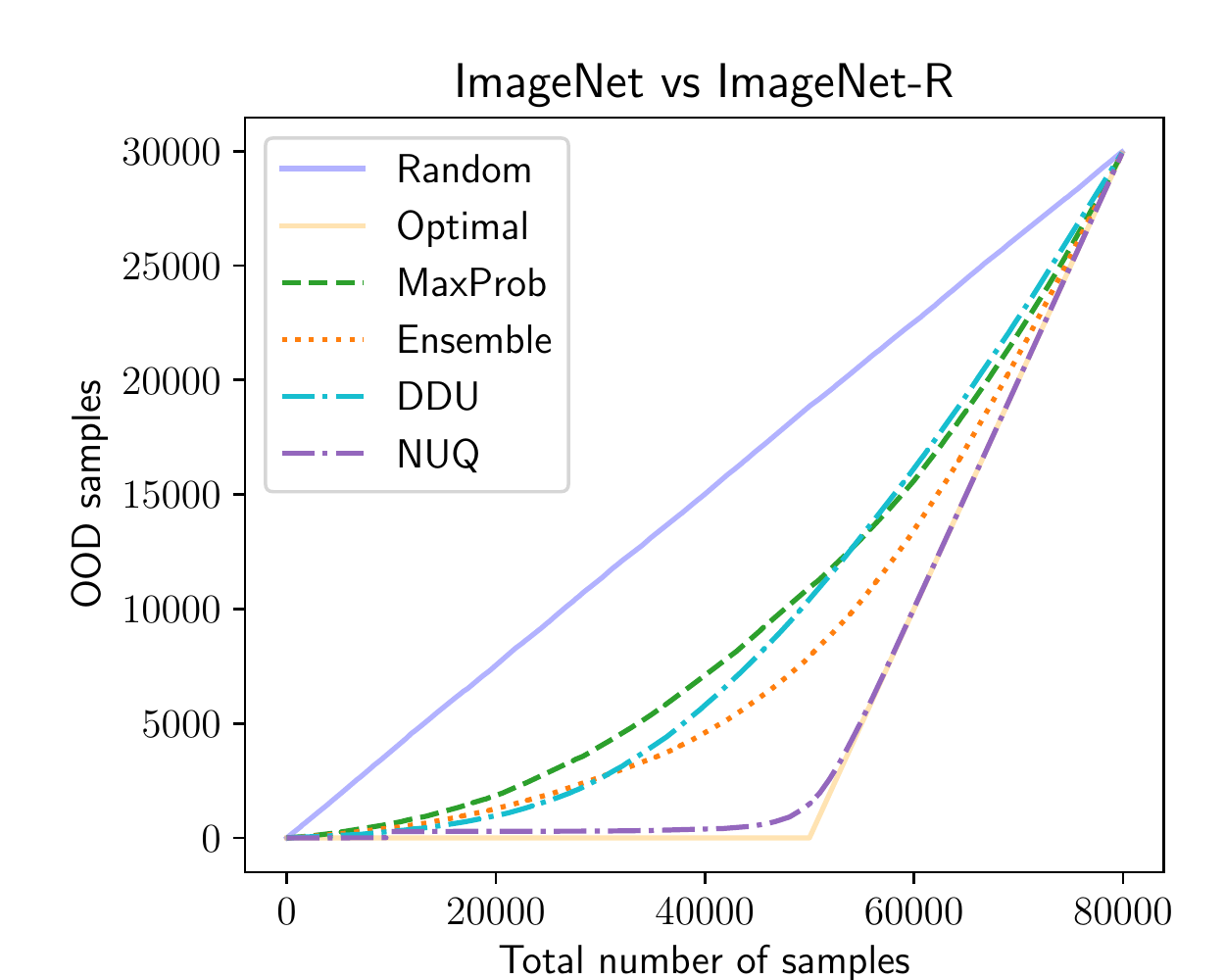}
        \caption{ImageNet-R as out-of-distribution}
        \label{subfig: nuq vs plug-in data distribution}
     \end{subfigure}
          \begin{subfigure}{0.48\textwidth}
        \centering
        \includegraphics[width=\textwidth]{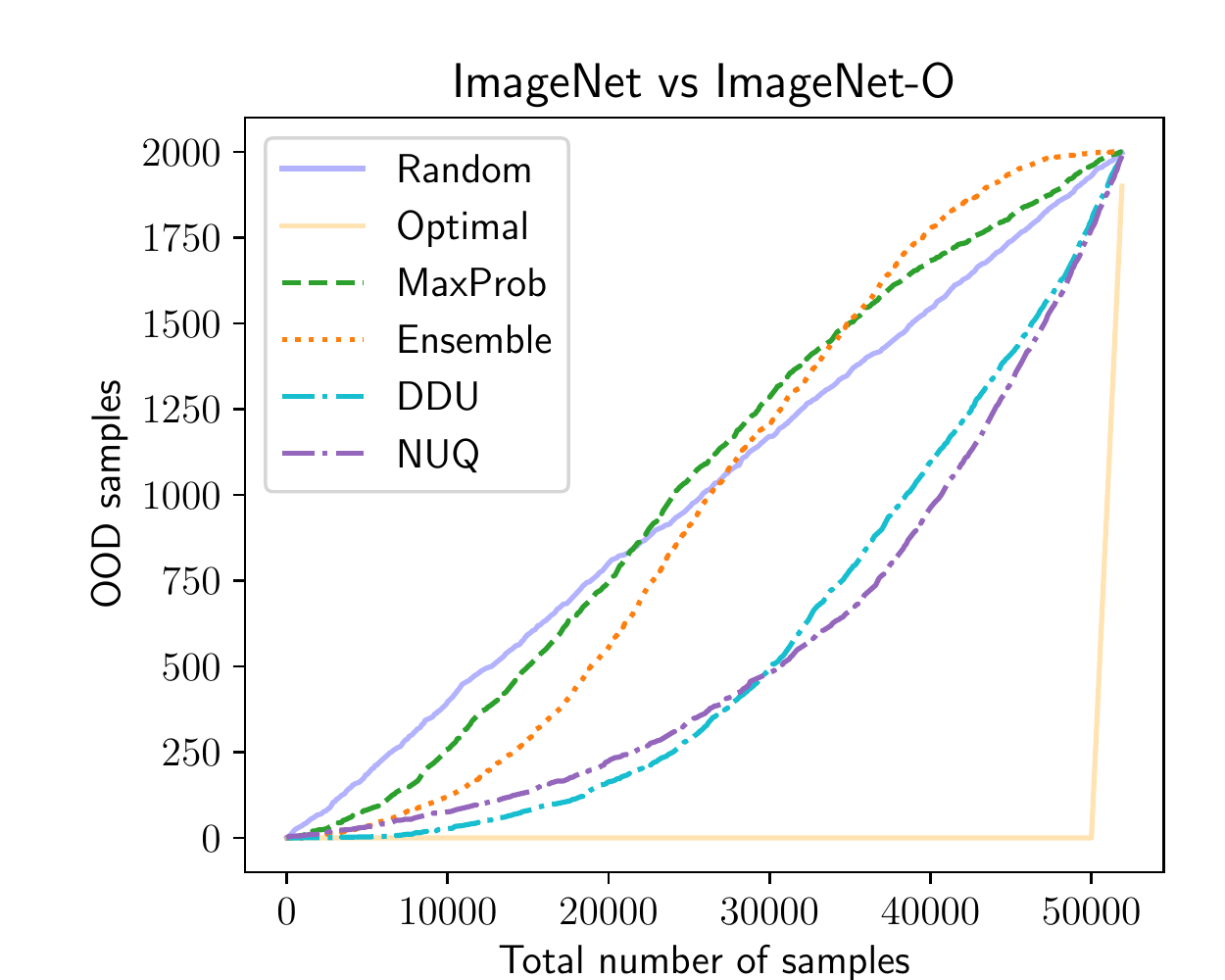}
        \caption{ImageNet-O as out-of-distribution}
        \label{subfig: nuq vs plug-in data distribution}
     \end{subfigure}
    \caption{Detecting OOD images. We sort images by increasing uncertainty and plot the number of in distribution points (lower is better). Ideal model should select samples from original dataset first.}
    \label{fig:mixed_imagenet_r}
  \end{figure}

\section{Hyperparameter Values and Dataset Statistics for Experiments on Textual Data}
\label{sec:textual_hyperparams}

For the optimal hyperparameter search in the experiments on textual datasets, we use Bayesian optimization with an early stopping algorithm. We divide the original training data into validation and training subsets in a ratio of 20 to 80 and perform optimization using sets of pre-defined values for each hyperparameter, which are given in the caption of Table \ref{tab:hyperparam}. As an objective metric, we use the accuracy score. The dataset statistics are presented in Table \ref{tab:ner_dataset}. 

  \begin{table*}[!ht]
    \resizebox{\textwidth}{!}{
    \begin{tabular}{|l|l|l|l|l|l|l|l|}
    \hline
    Dataset &  Objective Score & Spect. Norm. & SNGP & Learning Rate &  Num. Epochs &  Batch Size & Weight Decay \\ \hline 
     MRPC &            0.867 &            - &    - &          5e-05 &           12 &          32 &           1e-1 \\
       MRPC &            0.858 &            + &    - &          3e-05 &           11 &          32 &           1e-1 \\
       MRPC &            0.873 &            + &    + &         1e-4 &            5 &          16 &           0 \\\hline
       CoLA &             0.88 &            - &    - &          1e-05 &            8 &           4 &           1e-1 \\
       CoLA &            0.876 &            + &    - &          3e-05 &           15 &          32 &           1e-1 \\
       CoLA &            0.884 &            + &    + &          7e-06 &            5 &           8 &           0 \\\hline
      SST-2 &            0.936 &            - &    - &          1e-05 &           15 &          64 &           1e-1 \\
      SST-2 &            0.939 &            + &    - &          5e-05 &            7 &          64 &          1e-2 \\
      SST-2 &            0.921 &            + &    + &          2e-05 &           15 &           8 &           1e-1 \\\hline
      CLINC &            0.979 &            - &    - &          3e-05 &            7 &          16 &           1e-1 \\
      CLINC &             0.98 &            + &    - &          7e-05 &            9 &          64 &           1e-1 \\
      CLINC &            0.978 &            + &    + &          7e-05 &           13 &          64 &          1e-2 \\\hline
      ROSTD &            0.994 &            - &    - &          7e-06 &            6 &          32 &           1e-1 \\
      ROSTD &            0.995 &            + &    - &          7e-06 &            6 &          32 &           0 \\
      ROSTD &            0.994 &            + &    + &          3e-05 &           13 &          64 &           0 \\

    \hline
    \end{tabular}
    }
    \caption{\label{tab:hyperparam} Optimal hyperparameters for the experiments with ELECTRA on textual data. ``Objective score'' refers to the accuracy score for classification on the validation sample. We select hyperparameter values from the following pre-defined list: \\ 
    \textbf{Learning rate}: [5e-6, 6e-6, 7e-6, 9e-6, 1e-5, 2e-5, 3e-5, 5e-5, 7e-5, 1e-4]; \\ \textbf{Num. of epochs}: $\{n \in \mathbb{N} | 2 \leq n \leq 15\}$; \\ 
    \textbf{Batch size}: [4, 8, 16, 32, 64]; \\ 
    \textbf{Weight decay}: [0, 1e-2, 1e-1].}
  \end{table*}

  \begin{table}[ht!]
    \centering
    \resizebox{.45\textwidth}{!}{
    \begin{tabular}{|l|c|c|c|}
      \hline
      Datasets &          Train & 
      Test & 
      \# Labels \\
       \hline

       MRPC &    3.7K &   0.4K &  2 \\ 
       CoLA &    8.6K &   1.0K &  2 \\ 
       SST-2 & 67.3K & 0.9K & 2 \\ \hline
       CLINC & 15K & 5.5K & 150 \\ 
       ROSTD & 30.5K & 11.7K & 12 \\ \hline
       20 News Groups &    11.3K &   7.5K &  20 \\ 
       IMDB &    20K &   25K &  2 \\ 
       TREC-10 & 5.5K & 0.5K & 6 \\ 
       WMT-16 & 4500K & 3K & - \\ 
       Amazon & 207.4K & 29.6K & 5 \\ 
       MNLI &    392.7K &   9.8K &  3 \\ 
       RTE &    2.5K &   0.3K &  2 \\

      \hline
    \end{tabular}
    }
    \vspace{0.2cm}
    \caption{\label{tab:ner_dataset} Dataset statistics. The table presents the number of sequences for the training and test parts of the datasets. For the datasets from the GLUE benchmark (MRPC, CoLA, SST-2), we used the available validation set as the test set. For CLINC and ROSTD we present the size of the training part only for in-domain intents. From the last 7 datasets, we use only the test part as OOD instances.}
  \end{table}

\end{document}